\documentclass{article}

\usepackage[left=1in,bottom=1in,right=1in,top=1in]{geometry}
\usepackage[utf8]{inputenc} 
\usepackage{hyperref}       
\usepackage{url}            
\usepackage{booktabs}       
\usepackage{amsfonts}       
\usepackage{nicefrac}       
\usepackage{subcaption} 
\usepackage{natbib}
\usepackage{times}
\renewcommand{\paragraph}[1]{\noindent{\bf #1.}}

\usepackage{beton}
\usepackage{euler}
\usepackage[T1]{fontenc}

\usepackage{amsthm}
\usepackage{amsmath}
\usepackage{amsfonts}
\usepackage{amssymb}
\usepackage{booktabs}
\usepackage{mathtools}
\usepackage{hyperref}
\usepackage[ruled,linesnumbered]{algorithm2e}
\usepackage{xspace}
\usepackage[dvipsnames]{xcolor}
\usepackage[normalem]{ulem}
\usepackage{enumitem}
\setlist[itemize]{noitemsep, topsep=0pt, label=$\blacktriangleright$, leftmargin=*}
\usepackage{wrapfig}
\usepackage[color=Aquamarine,textsize=scriptsize]{todonotes}
\usepackage{capt-of}
\usepackage{comment}

\newcommand{\Real}{\mathbb{R}}
\newcommand{\HC}{\{0,1\}}
\newcommand{\fh}{\texttt{FlyHash}\xspace}
\newcommand{\sh}{\texttt{SimHash}\xspace}
\newcommand{\fbfc}{\texttt{FBFC}\xspace}
\newcommand{\fbf}{\texttt{FBF}\xspace}
\newcommand{\sbfc}{\texttt{SBFC}\xspace}

\newcommand{\ccone}{\texttt{CC1}\xspace}
\newcommand{\cc}{\texttt{CC}\xspace}
\newcommand{\lr}{\texttt{LR}\xspace}
\newcommand{\mlpc}{\texttt{MLPC}\xspace}
\newcommand{\knnc}{\texttt{$k$-NNC}\xspace}
\newcommand{\onennc}{\texttt{$1$-NNC}\xspace}
\newcommand{\sklearn}{\texttt{scikit-learn}\xspace}

\newcommand{\R}{{\mathbb R}}
\newcommand{\N}{{\mathbb N}}
\newcommand{\pr}{{\mathop {\rm Pr}}}
\newcommand{\E}{{{\mathbb E}}}

\newcommand{\argmin}{\mathop {\rm argmin}}

\newtheorem{theorem}{Theorem}
\newtheorem{lemma}[theorem]{Lemma}
\newtheorem{cor}[theorem]{Corollary}

\usepackage{listings}
\definecolor{codegreen}{rgb}{0,0.6,0}
\definecolor{codegray}{rgb}{0.5,0.5,0.5}
\definecolor{codepurple}{rgb}{0.58,0,0.82}
\definecolor{backcolour}{rgb}{0.95,0.95,0.92}

\lstdefinestyle{mystyle}{
    backgroundcolor=\color{backcolour},
    commentstyle=\color{codegreen},
    keywordstyle=\color{magenta},
    numberstyle=\tiny\color{codegray},
    stringstyle=\color{codepurple},
    basicstyle=\ttfamily\scriptsize,
    breakatwhitespace=false,
    breaklines=true,
    captionpos=b,
    keepspaces=true,
    numbers=left,
    numbersep=5pt,
    showspaces=false,
    showstringspaces=false,
    showtabs=false,
    tabsize=2
}

\title{Neural Neighborhood Encoding for Classification}

\author{%
  Kaushik Sinha\\
  Wichita State University, USA\\
  \texttt{kaushik.sinha@wichita.edu}
  \and
  Parikshit Ram\\
  IBM Research, USA\\
  \texttt{p.ram@gatech.edu}
}

\begin{document}
\maketitle
\begin{abstract}
Inspired by the fruit-fly olfactory circuit, the Fly Bloom
Filter~\citep{dasgupta2018neural} is able to efficiently summarize the
data with a single pass and has been used for novelty detection. We
propose a new classifier (for binary and multi-class classification)
that effectively encodes the different local neighborhoods for each
class with a per-class Fly Bloom Filter.  The inference on test data
requires an efficient \fh~\citep{dasgupta2017neural} operation followed
by a high-dimensional, but {\em sparse}, dot product with the
per-class Bloom Filters. The learning is trivially parallelizable.  On
the theoretical side, we establish conditions under which the
prediction of our proposed classifier on any test example agrees with
the prediction of the nearest neighbor classifier with high
probability. We extensively evaluate our proposed scheme with over
$50$ data sets of varied data dimensionality to demonstrate that the
predictive performance of our proposed neuroscience inspired
classifier is competitive the the nearest-neighbor classifiers and
other single-pass classifiers.

\end{abstract}
\section{Introduction: Neurally inspired data structure} \label{sec:intro}
Neural circuits in the fruit-fly appear to assess the novelty of an
odor in a two step process. Any odor is first assigned a ``tag'' that
corresponds to a small set of Kenyon Cells (KC) that get activated by
the odor. \citet{dasgupta2017neural} interpret this tag generation
process as a hashing scheme, termed \fh, where the tag/hash is
effectively a very sparse point a high dimensional space (2000
dimensions with $\sim 95\%$ sparsity). The tag (or rather a subset of
it) serves as input to a specific mushroom body output neuron (MBON),
the MBON-$\alpha'3$, where the response of this neuron to the odor
hash encodes the novelty of an odor. \citet{dasgupta2018neural}
``interpret the KC$\to$MBON-$\alpha'3$ synapses as a Bloom Filter''
that effectively ``stores'' all odors previously exposed to the
fruit-fly. This {\em Fly Bloom Filter} (\fbf) generates continuous
valued, distance and time sensitive novelty scores that have been
empirically shown to be highly correlated to the ground-truth novelty
scores relative to other Bloom Filter-based novelty scores for both
neural activity data sets (odors and faces) and vision data sets
(MNIST and SIFT). Theoretically, bounds on the expected novelty scores
of similar and dissimilar points have been established for binary and
exponentially distributed data.

In this paper, we propose a {\bf simple} extension of \fbf to binary
and multi-class classification, where we summarize each class with its
own \fbf and utilize the {\em familiarity} scores (inverse novelty
scores) from each class to label any test point. We theoretically
study why this simple idea works, and empirically demonstrate that the
simplicity does not preclude utility.
Specifically, we present
\begin{itemize}
\item A novel \fbf based classifier (\fbfc) that can be learned in an
  {\em embarassingly parallelized} fashion with a {\em single pass} of
  the training set, {\em provide insights} into the problem structure,
  and can be inferred from with an {\em
  efficient} \fh~\citep{dasgupta2017neural} followed by a {\em sparse}
  dot-product.
\item A theoretical examination of the proposed scheme, establishing
  conditions under which {\fbfc} {\em agrees} with the nearest-neighbor
  classifier.
\item A thorough empirical comparison of \fbfc to $k$-nearest-neighbor
  (\knnc) and other standard classifiers on over $50$ data sets from
  different domains.
\item A demonstration of the scaling of the parallelized \fbfc training
  process.
\item We present how the \fbfc can be used to interpret similarities
  between different classes in classification problem.
\end{itemize}
The paper is organized as follows: We discuss related work in
Section~\ref{sec:rel-work}. We detail our proposed algorithm in
Section~\ref{sec:algo} and analyze its theoretical properties in
Section~\ref{sec:theory}. We evaluate the empirical performance of
\fbfc against baselines in Section~\ref{sec:emp-evals} and conclude
with a discussion in Section~\ref{sec:conc}.


%
\section{Related work} \label{sec:rel-work}
Neuroscience inspired techniques are now widely accepted in artificial
intelligence to great success~\citep{hassabis2017neuroscience},
especially in the field of deep learning in the form of convolutional
neural
networks~\citep{kavukcuoglu2010learning,krizhevsky2012imagenet},
dropout~\citep{hinton2012improving} and attention
mechanisms~\citep{larochelle2010learning,mnih2014recurrent} to name a
few. Much like most machine learning methods, deep learning relies on
loss-gradient based training in most cases. In contrast, our proposed
\fbfc learning does not explicitly minimize any ``loss''
function. Moreover, rather than learning a representation for the
points that facilitates classification/regression, the \fbfc learns a
representation for entire classes, allowing test points to be compared
to classes for computing familiarity scores.

Given the correlation between a point $x$'s \fbf novelty score to its
minimum distance from the set that the \fbf
summarizes~\citep{dasgupta2018neural}, our proposed neuroscience
inspired \fbfc is perhaps closest to the nonparametric
$k$-nearest-neighbor classifier (\knnc). Vanilla \knnc does not have
an explicit loss or a training phase given a measure of similarity;
all the computation is shifted to inference. \fbfc does have an
explicit training phase, but requires only a single pass of the
training data -- once a point is processed into the \fbf, it can be
discarded, making \fbfc suitable for streaming data.

On a very high level, this is similar to cluster-based \knnc where
class specific training data (data with same labels) is summarized as
(multiple) cluster centers and used as a reduced training set on which
\knnc is applied. A variety of methods exists in literature that adopt
this simple idea of data reduction~\citep{zhou2010clustering_knn,
  parvin2012cluster,oigiaroglou2013homogeneous,oigiaroglou2016rhc,
  gallego2018cluster,gou2019localmean}. These algorithms are designed
with the goal of reducing the high computational \& storage
requirements of \knnc. Orthogonally, various data structures have been
utilized to accelerate the nearest-neighbor search in \knnc inference
representing the data as an index such as space-partitioning
trees~\cite{omohundro1989fiveballtree,
  beygelzimer2006cover,dasgupta2015rpt,ram2019kdtree} and hash tables
generated by {\em locality-sensitive}
hashes~\cite{gionis1999hashing,andoni2008lsh}.

The closely related locality-sensitive Bloom filter (LSBF)
\cite{kirsch2006distance_bloom,hua2012ls_bloom} also summarizes the
data similar to \fbf, relying on distance preserving random projection
\cite{vempala2004rp} to lower dimensionalities followed by quantizing
the projected vector to an integer. Under this scheme, two inputs
reset the same bit in the filter if they are assigned the exact same
projected vector. Performance of LSBF heavily depends on the choice of
hyperparameters that control the projection dimensionality and the
data-independent quantization scheme. \fbf has been shown to be
empirically outperform LSBF for novelty detection.

Multinomial regression with linear models and multi-layered perceptron
can also be viewed as learning a set of weight vectors corresponding
to each class, with the inner product of the test point with these
vectors driving the class assignment.


%
\section{\fh Bloom Filter Classifier (\fbfc)} \label{sec:algo}
The basic building block of our proposed algorithm is a fruit-fly
olfactory circuit inspired \fh function, first introduced by
\citet{dasgupta2017neural}. Here we consider the binarized
\fh~\citep{dasgupta2018neural}. For $x \in \Real^d$, the \fh function
$h \colon \Real^d \to \HC^{m}$ is defined as,
\begin{equation} \label{eq:flyhash}
    h(x) = \Gamma_\rho (M_{m}^s x),
\end{equation}
where $M_m^s \in \HC^{m \times d}$ is the randomized sparse lifting
binary matrix with $s \ll d$ nonzero entries in each row, and
$\Gamma_\rho \colon \Real^{m} \to \HC^{m}$ is the winner-take-all
function converting a vector in $\Real^m$ to one in $\HC^m$ by setting
the highest $\rho \ll m$ elements to $1$ and the rest to
zero\footnotemark.
\footnotetext{\fh~\cite{dasgupta2017neural} leaves the highest $\rho$
  elements as is and sets the rest to zero, but requires each
  $x \in \Real^d$ to be mean-centered ($\sum_{i=1}^d x_i = 0$). \fbf
  needs a binarized \fh \citep{dasgupta2018neural}, where
  mean-centering is redundant.}
Unlike random projection~\cite{vempala2004rp} which decreases data
dimensionality after projection, \fh is an upward projection which
increases data dimensionality ($m \gg d$). The hyper-parameters for
\fh are (i) the projected dimensionality $m \in \N$, (ii) projection
matrix nonzero count per row $s \in \N$, and (iii) the number of
nonzeros (NNZ) $\rho \in \N$ in the \fh.  The run time for \fh is
$O(ms + m\log \rho)$. The \fh function can also be viewed as a maximum
inner product search problem \citep{ram2012maximum,
  shrivastava2014asymmetric} where we seek the $\rho$ rows in
$M_{m}^s$ with the highest inner-product to $x$ and sped up using fast
algorithms.
\begin{figure}[t]
  \centering
  \begin{subfigure}{0.59\textwidth}
    \frame{\includegraphics[width=\textwidth]{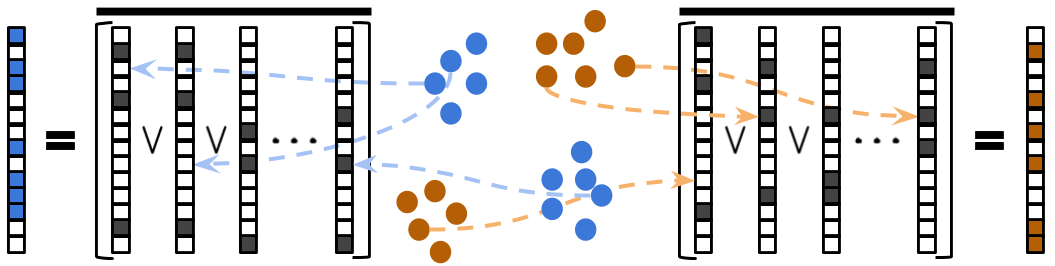}}
    \caption{\fbfc training.}
    \label{fig:fbfc-train}
  \end{subfigure}
  ~
  \begin{subfigure}{0.34\textwidth}
    \frame{\includegraphics[width=\textwidth]{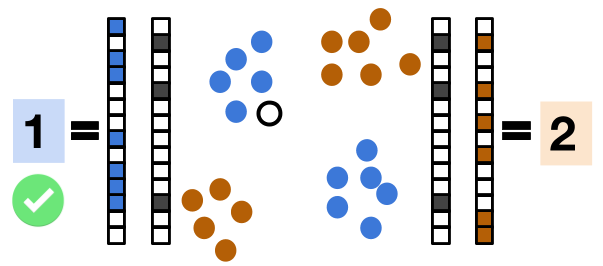}}
    \caption{\fbfc inference.}
    \label{fig:fbfc-infer}
  \end{subfigure}
  \caption{Visual depiction of \fbfc training
    (Algorithm~\ref{alg:train-infer}) and inference
    (Equation~\eqref{eq:fbfc-predict}) for \fbfc.  Colored circles
    correspond to the labeled training set. In
    Figure~\ref{fig:fbfc-train}, the high dimensional sparse {\fh}es
    for the points (stacked $\blacksquare$ \& $\square$) in each class
    are used to generate the per-class \fbf (\texttt{NOT}
    $\overline{(\cdot)}$ of the \texttt{OR}s $\vee$ of the hashes as
    per Equation~\eqref{eq:fbf_rewrite}). The $\bigcirc$ in
    Figure~\ref{fig:fbfc-infer} is the unlabeled point we infer on
    based on the dot-product of its \fh with each of the per-class
    {\fbf}s (eliding the denominator in
    Equation~\eqref{eq:fbfc-predict}). {\em Please view in color.}}
    \label{fig:fbfc-viz}
\end{figure}

Using \fh as an algorithmic building block, \citet{dasgupta2018neural}
construct a \fbf to succinctly summarize the data, and use it to
effectively solve the unsupervised learning task of novelty
detection. Here we extend the use of \fbf to classification, an
instance of supervised learning. Specifically, we use \fbf to
summarize each class separately -- the per-class \fbf encodes the
local neighborhoods of each class, and the high dimensional sparse
nature of \fh (and consequently \fbf) summarizes classes
with multi-modal distributions while mitigating overlap between the
{\fbf}s of other classes.

\begin{wrapfigure}{R}{0.5\textwidth}
{
\begin{algorithm}[H]
\DontPrintSemicolon
\SetAlgoLined
\caption{\fbfc training with training set $S \subset
  \Real^d \times [L]$, projected dimensionality $m \in \N$, NNZ for
  each row in the projection matrix $s \ll d$, NNZ in the \fh $\rho
  \ll m$.
  }
\label{alg:train-infer}
\SetKwProg{train}{Function Train\fbfc: $(S, m, \rho, s) \to (M_m^s, \{
  w_i, i \in [L] \})$}{}{end}
\train{}{
  Initialize $w_1,\ldots,w_L \gets \mathbf{1}_m \in \HC^m$ \;
  $M_m^s \in \HC^{m \times d}$ \tcp*[r]{$s$ NNZ per row}
  \For{$(x,y)\in S$}{
    $h(x) \gets \Gamma_\rho (M_m^s x)$ \;
    $w_y \gets w_y \bigwedge \overline{h(x)}$ 
  }
  \KwRet{$ (M_m^s, \{ w_i, i \in [L] \}) $} \;
}
\end{algorithm}
}
\end{wrapfigure}

\paragraph{\fbfc training}
Let $w_i \in\{0,1\}^m$ be the \fbf for any class $i \in [L] = \{1, 2,
\ldots, L\}$, with $w_i$ initialized to $\mathbf{1}_m \in \HC^m$, the
all one vector.  For any point $x \in \Real^d$ with label $y = i$ in
the training set $S \subset \Real^d \times [L]$, $w_i$ is updated with
the \fh $h(x)$ as follows -- the bit positions of $w_i$ corresponding
to the nonzero bit positions of $h(x)$ are set to zero, represented as
$w_i \gets (w_i \oplus h(x)) \wedge w_i = w_i \wedge \overline{h(x)}$,
where $\oplus$, $\wedge$ and $\overline{(\cdot)}$ are the
\texttt{XOR}, \texttt{AND} and \texttt{NOT} operators
respectively. Starting with $w_i = \mathbf{1}_m$, the updates for any
two examples $(x_1,y_1), (x_2,y_2)\in S$ with $y_1 = y_2 = i$ can be
succinctly written as $ w_i \gets
\overline{\overline{\mathbf{1}_m}\vee h(x_1)\vee h(x_2)}$ with the
application of De Morgan's law, with $\vee$ as \texttt{OR}. We can now
condense the \fbf construction for a class $i \in [L]$ to
\begin{align} \label{eq:fbf_rewrite}
  w_i
  = \overline{\overline{\mathbf{1}_m}\bigvee_{(x,y) \in S \colon y = i} h(x)}
  = \overline{\bigvee_{(x,y) \in S \colon y = i} h(x)}
  = \mathbf{1}_m\bigwedge_{(x,y) \in S \colon y = i} \overline{h(x)}
  = \bigwedge_{(x,y) \in S \colon y = i} \overline{h(x)}.
\end{align}
This new interpretation makes the \fbf construction trivially
parallelizable -- $w_i$ for each $i \in [L]$ can be computed either by
a series of commutative \texttt{OR}s followed by a \texttt{NOT} at the
end or by a series of commutative \texttt{AND}s, and the process is
order-independent. The $L$ per-class {\fbf}s (and the lifting matrix
$M_m^s$) constitute our proposed {\em
  \fbfc}. Algorithm~\ref{alg:train-infer} (Train\fbfc) presents the
\fbfc training, and Figure~\ref{fig:fbfc-train} visualizes the process
for a toy example.

\paragraph{\fbfc inference}
For a point $x$, we compute the per-class novelty scores $f_i(x) \in
[0,1], i \in [L]$ and the predicted label as:
\begin{equation} \label{eq:fbfc-predict}
  f_i(x) =\left(w_i^\top h(x) \right) / \rho, \quad
  \hat{y} = \arg\min_{i \in [L]} f_i(x)
\end{equation}
A high $f_i(x)$ indicates that majority of the training examples with
label $i$ are very different from $x$.
A small value of $f_i(x)$ indicates the existence of at least one
training example with label $i$ similar to $x$. The predicted label
for $x$ is simply the class with the smallest $f_i(x)$ (breaking ties
randomly). This is visualized in Figure~\ref{fig:fbfc-infer}.  The
per-class $f_i(x), i \in [L]$ can be converted into class
probabilities with a soft-max operation.

\paragraph{Computational complexities}
\fbfc training time with $n$ points is $O(nms + nm \log \rho + n
\rho)$ for the $n$ \fh operations, followed by $n$ \texttt{OR}s with
the class-specific {\fbf}s.  The commutative \texttt{OR} operator
allows us to chunk $n$ points across $T$ threads for parallel
processing of groups of size $n/T$ -- in a shared memory setting, all
threads operate on the same set of per-class {\fbf}s, resulting in a
$O\left(\frac{n}{T} \left(ms + m \log \rho + \rho \right)\right)$
runtime, demonstrating linear scaling with $T$.  In distributed memory
setting, each process operates its own set of {\fbf}s that are finally
{\em all-reduced} in additional $O(m L \log T)$ time.
Memory overhead for training with a batch of $n'$ training point is
$O(ms + n'm + mL)$. The batch size can be as small as $1$, implying a
minimum memory overhead during training of $O(m(s+L))$. If $T$ threads
are processing batches of size $n'$ each, the memory overhead
increases linearly with $T$.
\fbfc inference takes $O(ms + m \log \rho + L \rho)$ per
point. However, the inference problem $\min_{i \in L} w_i^\top h(x)$
can be reduced to a maximum inner product search
problem~\citep{ram2012maximum, shrivastava2014asymmetric} and solved
in time sublinear in $L$ for large $L$.
%
%

\paragraph{Inter-class similarities}
Given the per-class {\fbf}s $w_i, i \in [L]$, we propose the cosine
similarity $s_{ij}$ between the \fbf pair $(w_i, w_j)$ as a similarity
score between classes $i$ and $j$ to quantify the hardness of
differentiating these classes, and provides an insight into the
structure of the classification problem.

\paragraph{Non-binary \fbf}
In our binary \fbf design, for any test point $x$ and any $i\in [L]$,
let $A_x = \{j \colon (h(x))_j = 1 \}$ be the nonzero coordinates in
$h(x)$. Each coordinate of $A_x$ contributes in deciding the value of
$f_i(x)$. For any $j\in A_x$, it is possible that a single training
example $x'$ from class $i$ sets the contribution of the
$j^{\mbox{th}}$ coordinate to zero in the computation of $f_i(x)$ --
it is only required that $h((x'))_j = 1$; since $h$ is randomized,
there is always a nonzero probability of this event. Also, for any
$j,k\in A_x, j\neq k$, if $w_{ij}=w_{ik}=0$ (the $j^{\text{th}}$ and
$k^{\text{th}}$ element in the \fbf for class $i$), coordinates $j$
and $k$ are indistinguishable in terms of their contribution to
$f_i(x)$. To address these limitations, we present a modified \fbf
design which aims to capture neighborhoods and distribution
information more effectively, by allowing coordinates of $w_i$ to take
value in $[0,1]$. In this design, for any fixed $c\in (0,1]$, the
$j^{\mbox{th}}$ coordinate of $w_i$ is set as follows, with $c = 1$
corresponding to binary \fbf:
\begin{equation} \label{eq:nb-fbf-agg}
    w_{ij}=(1-c)^{\left| \left\{(x,y)\in S \colon y=i \mbox { and }
      (h(x))_j=1 \right\} \right|},
\end{equation}
The label for a test point $x\in\mathbb{R}^d$ is still
computed as $\hat{y}=\arg\min_{i \in [L]} w_i^{\top}h(x)$. We term
this form of the Fly Bloom Filter as $\fbf^*$ and the corresponding
classifier as $\fbfc^*$. For any $i,j$, since $w_{ij}$ is computed by
counting the number of examples $(x',y')\in S$ satisfying $y'=i \mbox
{ and } (h(x'))_j=1$, and raising this count to the power of $(1-c)$,
$\fbf^*$ is still equally parallelizable as the binary \fbf -- the
\texttt{OR} aggregation followed by a \texttt{NOT} is now instead
a (sparse) summation over the {\fh}es, followed by an
exponentiation of $(1 - c)$.  The exponential decay in
equation~\eqref{eq:nb-fbf-agg} allows $w_{ij}$ to be
determined by a local neighborhood of size dependent on $c$. We
discuss this further in Supplement S1.
%


%
\section{Theoretical analysis} \label{sec:theory}
In this section we present theoretical analysis of \fbfc, identifying
conditions under with \fbfc agrees with the nearest-neighbor
classifier \onennc. First we describe the general setup and present
our generic analysis when certain abstract conditions are
satisfied. Then we consider two special cases that are different
instantiations of this generic result. All proofs are deferred to
Supplement S2.
\subsection{Preliminaries}
We denote a single row of a projection matrix $M_m^s$ by
$\theta\in\{0,1\}^d$ drawn i.i.d. from $Q$, the uniform distribution
over all vectors in $\{0,1\}^d$ with exactly $s$ ones, satisfying
$s\ll d$. For ease of notation, we use $M$ instead of $M_m^s$ and we
use an alternate formulation of the winner-take-all strategy as
suggested in~\cite{dasgupta2018neural}, where for any $x\in\R^d,
\tau_x$ is a threshold that sets largest $\rho$ entries of $Mx$ to one
(and the rest to zero) in expectation.
Specifically, for a given $x\in\R^d$ and for any fraction $0<f<1$, we
define $\tau_x(f)$ to be the top $f$-fractile value of the
distribution $\theta^{\top} x$, where $\theta\sim Q$:
\begin{equation}
    \tau_x(f)=\sup\{v: \pr_{\theta\sim Q}(\theta^{\top} x\geq v)\geq f\}
\end{equation}
We note that for any $0<f<1$, $\pr_{\theta\sim Q}(\theta^{\top} x\geq
\tau_x(f))\approx f$, where the approximation arises from possible
discretization issues. For convenience, henceforth we will assume that
this is an equality:
\begin{equation}
    \pr_{\theta\sim Q}(\theta^{\top} x\geq \tau_x(f))= f
\end{equation}
For any two $x,x'\in\R^d$, we define: $q(x,x')=\pr_{\theta\sim Q}
\left(\theta^{\top} x'\geq \tau_{x'}\left( \rho / m \right) ~|~
\theta^{\top} x \geq\tau_x\left( \rho / m \right) \right)$.
This can be interpreted as follows: with $h(x), h(x')$ as the $\fh$es
of $x$ and $x'$, respectively, $q(x,x')$ is the probability that
$(h(x'))_j=1$ given that $(h(x))_j=1$, for any specific $j$.

We analyze classification performance of \fbfc trained on a training
set $S=\{(x_i,y_i)\}_{i=1}^{n_0+n_1}\subset \mathcal{X} \times
\{0,1\}$, where $S=S^{1}\cup S^{0}$, $S^{0}$ is a subset of $S$ having
label 0, and $S^{1}$ is a subset of $S$ having label 1, satisfying
$|S^{0}|=n_0$, $|S^{1}|=n_1$ and $n=\max\{n_0,n_1\}$. For appropriate
choice of $m$, let $w_{0}, w_{1}\in\{0,1\}^m$ be the $\fbf$s
constructed using $S^{0}$ and $S_1$ respectively.

\subsection{Connection to \onennc}
Without loss of generality, for any test example $x\in\mathcal{X}$,
assume that its nearest neighbor from $S$ has class label 1. Then
\onennc will predict $x$'s class label to be 1. With $h(x)$ as the \fh
of $x$ (equation~\ref{eq:flyhash}), if we are able to show that
$\E_M(w_{1}^{\top} h(x))< \E_M(w_{0}^{\top}h(x))$ then \fbfc will
predict, in expectation, $x$'s label to be 1.  The following lemma
quantifies the expectation of class specific novelty scores and their
upper and lower bounds.
\begin{lemma}\label{lem:expectation}
Fix any $x\in\R^d$ and let $h(x)\in\{0,1\}^m$ be its \fh using
equation~\ref{eq:flyhash}. Let $x_{NN}^{i}=\argmin_{(x',y')\in
  S^{i}}\|x-x'\|$ for $i\in\{0,1\}$, where $\|\cdot\|$ is any distance
metric. Let $A_{S^1}=\{\theta: \cap_{(x',y')\in S^{1}} ~ \theta^{\top}
x'<\tau_{x'}(\rho/m)\}$ and $A_{S^0}=\{\theta: \cap_{(x',y')\in S^{0}}
~ \theta^{\top} x'<\tau_{x'}(\rho/m)\}$. Then the following holds,
where the expectation is taken over the random choice of projection
matrix $M$.\\
(i) $\E_M(\frac{w_{1}^{\top} h(x)}{\rho})=\pr_{\theta\sim
  Q}\left(A_{S^1} | \theta^{\top}x\geq \tau_x(\frac{\rho}{m}\right)$,
(ii) $\E_M(\frac{w_{0}^{\top}h(x)}{\rho})=\pr_{\theta\sim
  Q}\left(A_{S^0} | \theta^{\top}x\geq
\tau_x(\frac{\rho}{m})\right)$\\
(iii) $\E_M(\frac{w_{1}^{\top} h(x)}{\rho})\geq 1-\sum_{x'\in
  S^{1}}q(x,x')$,
(iv) $\E_M(\frac{w_{1}^{\top} h(x)}{\rho})\leq 1-q(x,x_{NN}^1)$\\
(v) $\E_M(\frac{w_{0}^{\top} h(x)}{\rho})\geq 1-\sum_{x'\in
  S^{0}}q(x,x')$,
(vi) $\E_M(\frac{w_{0}^{\top} h(x)}{\rho})\leq 1-q(x,x_{NN}^0)$
\end{lemma}
This immediately provides us a sufficient condition for \fbfc to agree
with \onennc on any test point $x$ in expectation -- the upper bound
of $\E_M ( w_{1}^{\top} h(x))$ should be strictly smaller than lower
bound of $\E_M (w_{0}^{\top} h(x))$.
\begin{theorem}\label{th:general}
Fix any $\delta\in(0,1)$, $s\ll d$, and $\rho\ll m$. Given a training
set $S$ as described above and a test example $x\in\mathcal{X}$, let
$x_{NN}$ be its closest point from $S$ measured using $\ell_p$ metric
for an appropriate choice of $p$. If (i)
$\rho=\Omega(\log(1/\delta))$, (ii) $\|x-x_{NN}\|_p=O(1/s)$, and (iii)
$m=\Omega(n\rho)$, then under mild conditions, with probability at
least $1-\delta$ (over the random choice of projection matrix $M$),
prediction of \fbfc on $x$ agrees with the prediction of 1-NN
classifier on $x$.
\end{theorem}
\noindent{\em Proof (sketch).} If either the structure of
$\mathcal{X}$ allows us to choose a threshold $\tau_x$ that is
identical for any $x\in\mathcal{X}$, resulting in a closed form
solution for the quantity $q(x,x')$ for any $x,x'\in\mathcal{X}$, or
the distributional assumption on $\mathcal{X}$ sets the quantity $\E_x
q(x,x')$ to be identical for all $x'\in S$, then all the three
conditions mentioned in theorem are satisfied. This property, in
conjunction with Lemma~\ref{lem:expectation}, yields the desired
result in expectation under mild conditions. The high probability
result then follows using standard concentration bounds.

\paragraph{Multi-class classification} The above results can be
extended to multi-class classification problem involving $L$ classes
in a straight forward manner by applying concentration result to each
of the $\left((w_i^{\top}h(x))/\rho\right)$, for $i\in [L]$, and using
a union bound (see Supplement S2.4).

Note that the \fbf guarantees for novelty detection are limited to two
special cases: (i) examples with binary feature vectors containing
fixed number of ones, and (ii) examples sampled from a permutation
invariant distribution \cite{dasgupta2018neural}. We extend this
analysis with these two cases to provide guarantees for \fbfc in
multi-class classification, which is a distinct learning problem from
novelty detection.

\subsection{Special case I: Binary data}
In this section we consider a special case where examples from each
class have binary feature vectors with fixed number of ones. In
particular, let $\mathcal{X}=\mathcal{X}_b=\{x\in\HC^d : |x|_1 = b <
d\}$.
\begin{theorem}\label{th:binary_new}
Let $S$ be a training set as given above. Fix any $\delta\in(0,1)$,
and set $\rho\geq\frac{12}{\mu}\ln(4/\delta)$, $m\geq (d/b)n\rho$, and
$s=\log_{d/b}(m/\rho)$, where $\mu=\min\left\{\E_M\left((w_0^{\top}
h(x))/\rho\right),\E_M\left((w_1^{\top} h(x))/\rho\right)\right\}$ and
$h(x)$ is the \fh (eq.~\eqref{eq:flyhash}). For a test point
$x\in\mathcal{X}$, let its closest point from $S$ measured using
$\ell_1$ metric be $x_{NN}$, having label $y_{NN}\in\{0,1\}$,
satisfies, (i) $\|x-x_{NN}\|_1\leq 2b(1-b/d) /3s$, and (ii)
$\|x-x_i\|_1\geq 2b(1-b/d)$ for all $(x_i,y_i)\in S$, with $y_i\neq
y_{NN}$. Let $w_{0}, w_{1}\in\{0,1\}^m$ be the {\fbf}s constructed
using $S^{0}$ and $S^{1}$ respectively. Then, with probability at
least $1-\delta$ (over the random choice of projection matrix $M$),
\fbfc prediction on $x$ agrees with the \onennc prediction on $x$.
\end{theorem}
Here $s=O(\log n)$, which is the same logarithmic dependence that was
also established in \cite{dasgupta2018neural}.
\subsection{Special Case II: Permutation invariant distribution in $\Real^d$}
Here we show that, for \emph{permutation invariant} distributions,
\fbfc agrees with \onennc in $\Real^d$ with high
probability. Permutation invariant distribution in the \fbf context
was introduced in \citet{dasgupta2018neural} and defined as a
distribution $P$ over $\Real^d$ 
permutation $\sigma$ of $\{1,2,\ldots,d\}$ and any
$x=(x_1,\ldots,x_d)\in\Real^d$,
$P(x_1,\ldots,x_d)=P(x_{\sigma(1)},\ldots,x_{\sigma(d)})$.  Precisely,
we show
%
\begin{theorem}\label{th:real_new}
Let $S$ be a training set as given above. Fix any $\delta\in(0,1)$,
$s\ll d$, and set $\rho\geq\frac{48}{\mu}\ln(8/\delta)$ and $m\geq
14n\rho/\delta$, where $\mu=\min\left\{\E_M\left((w_0^{\top}
h(x))/\rho\right),\E_M\left((w_1^{\top} h(x))/\rho\right)\right\}$,
$h(x)$ is the \fh (eq.~\eqref{eq:flyhash}), and $w_{0},
w_{1}\in\{0,1\}^m$ are the {\fbf}s constructed using $S^{0}$ and
$S^{1}$ respectively. For a test point $x\in\Real^d$, sampled from a
permutation invariant distribution, let $x_{NN}$ be its nearest
neighbor from $S$ measured using $\ell_{\infty}$ metric, which
satisfies $\|x-x_{NN}\|_{\infty}\leq \Delta/s$, where
$\Delta=\frac{1}{2}\left(\tau_x(2\rho/m)-\tau_x(\rho/m)\right)$ and
has label $y_{NN}\in\{0,1\}$. Then, with probability at least
$1-\delta$ (over the random choice of projection matrix $M$), \fbfc
prediction on $x$ agrees with \onennc prediction on $x$.
\end{theorem}
\paragraph{Towards $\Real^d$}
The structure of the binary and permutation-invariant distributions
allow us to get these novel, yet limited, result. Similar results for
general $\R^d$ are more challenging and non-trivial -- for any
$x,x'\in\R^d, x\neq x'$, the thresholds $\tau_x$ and $\tau_{x'}$ will
be different and a closed form solution for $q(x,x')$ may not exist,
and we need to find explicit bounds for this quantity. Our hypothesis
is that we will need various data dependent assumptions, including
smoothness of conditional probability function and Tysbakov's margin
conditions~\citep{tsybakov2004optimal,audibert2007fast}, to get a
similar result for $\R^d$.


%
\section{Empirical evaluations} \label{sec:emp-evals}
In this section, we evaluate the empirical performance of
\fbfc. First, we evaluate the dependence of \fbfc on its
hyper-parameteres. Then, we compare \fbfc to other classifiers that
can be trained in a single pass on (i) synthetic data, (ii) OpenML
(binary \& multi-class) classification data
sets~\citep{van2013openml}, and (iii) 4 popular vision data sets --
{\sc Mnist}, {\sc Fashion-Mnist}, {\sc Cifar10}, {\sc
  Cifar100}. Finally, we study the computational scaling of the
parallelized \fbfc training and present some problem insights
generated by a trained \fbfc.  The details on the implementation and
compute resources are in Supplement S3.
\subsection{Dependence on hyper-parameters} \label{sec:emp-eval:hpdep}
We study the effect of the different \fbfc hyper-parameters: (i) the
\fh dimension $m$, (ii) the NNZ per-row $s \ll d$ in $M_m^s$, (iii)
the NNZ $\rho$ in the \fh, and (iv) the \fbf decay rate $c$. We
consider $6$ OpenML data sets (see Table S1 in Supplement S3 for data
details). For every hyper-parameter setting, we compute the $10$-fold
cross-validated classification accuracy ($1 - $ misclassification
rate). We vary each hyper-parameter while fixing the others. The
results for each of the hyper-parameters and data sets are presented
in Figures S1 \& S2 in Supplement S3.1.

The results indicate that, for fixed $\rho$, increasing $m$ usually
improves \fbfc performance up to a point. \fbfc performance is not
affected by $s$ for the high dimensional sets; for the lower
dimensional sets ($d < 20$), the performance improves with increasing
$s$ till around $s \approx 10$, after which, the performance
degrades. Increase in $\rho$ improves \fbfc performance for fixed
values of $m$ and other hyper-parameters. The \fbfc performance is not
affected much by the value of the decay rate when $c < 1$, but there
is a significant drop in performance as we move from $c < 1$
(non-binary \fbf) to $c = 1$ (binary \fbf), indicating the advantage
of our novel non-binary \fbf; this behavior is pretty consistent and
obvious across all data sets. See Supplement S3.1 for further details
and discussion.

\subsection{Comparison to baselines} \label{sec:emp-eval:comp2baselines}
We compare our proposed \fbfc to various baselines. Given the
significant difference between \fbfc with $c = 1$ (binary Bloom
Filter) and \fbfc with $c < 1$, we consider both cases, with $\fbfc^*$
explicitly denoting $c < 1$. We evaluate the proposed schemes and all
the baselines relative to the $k$-nearest-neighbor classifier (\knnc).
We consider a variety of baselines, including ones that can be trained
in a single pass of the training data (similar to \fbfc):
\begin{itemize}
\item {\bf \knnc:} This is the primary baseline. We tune over the
  neighborhood size $k \in [1,64]$.
\item {\bf \ccone:} We consider classification based on a single
  prototype per class -- the geometric center of the class, computed
  with a single pass of the training set.
\item {\bf \cc:} This generalizes \ccone where we utilize multiple
  prototypes per class -- a test point is assigned the label of its
  closest prototype. The per-class prototypes are obtained by
  $k'$-means clustering. We tune over the number of clusters per-class
  $k' \in [1, 64]$. This is {\em not single pass}.
\item {\bf \sbfc:} We utilize \sh~\citep{charikar2002similarity} based
  LSBF for each class in place of \fbf to get the \sh Bloom Filter
  classifier (\sbfc).  We consider this to demonstrate the need for
  the high level of sparsity in \fh; \sh is not inherently as sparse.
  We tune over the \sh projected dimension $m$, considering $m < d$
  (traditional) and $m > d$ (as in \fh).  For the same $m$, \sh is
  more costly than \fh, involving a dense matrix-vector product
  instead of a sparse matrix-vector one.
\item {\bf \lr.} We consider logistic regression trained for a single
  epoch with a stochastic algorithm and tune over $960$
  hyper-parameter configurations for each data set.
\item {\bf \mlpc.} We consider a multi-layer perceptron trained for a
  single epoch with the ``Adam'' solver~\citep{kingma2014adam} and
  tune over $288$ hyper-parameter configurations for each data set.
\end{itemize}
The complete details of the baselines and their hyper-parameters are
in Supplement S3.2.

\paragraph{\fbfc hyper-parameters}
For a data set with $d$ dimensions, we tune across $60$
hyper-parameter settings in the following ranges: $m \in [2d, 2048d]$,
$s \in (0.0, 0.5d]$, $\rho \in [8, 256]$, and $c \in [0.2, 1]$, with
$c = 1$ as binary \fbfc. We use this hyper-parameter search space
for all experiments, except for the vision sets, where we use $m \in
[2d, 1024d]$.

\paragraph{Evaluation metric}
For all methods (baselines and \fbfc), we compute the relative
performance on each data set as $(1 - a_M/a_k)$ where $a_k$ is the
best 10-fold cross-validated classification accuracy achieved by \knnc
and $a_M$ is the best 10-fold cross-validated classification accuracy
obtained by candidate method $M$ across different
hyper-parameters. \knnc has a relative performance of $0$.
\begin{figure}[t]
  \centering
  \begin{subfigure}{0.23\textwidth}
    \centering
    \includegraphics[width=\textwidth]{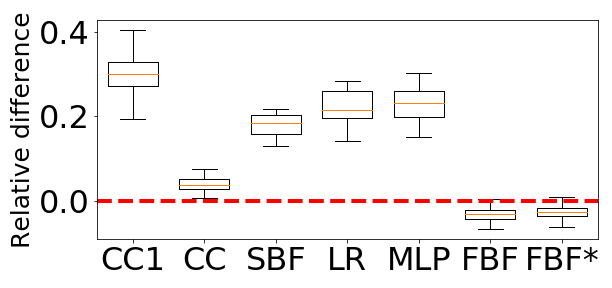}
    \caption{Syn. $\HC^{100}, b = 20$}
    \label{fig:bcs-b100}
  \end{subfigure}
  ~
  \begin{subfigure}{0.23\textwidth}
    \centering
    \includegraphics[width=\textwidth]{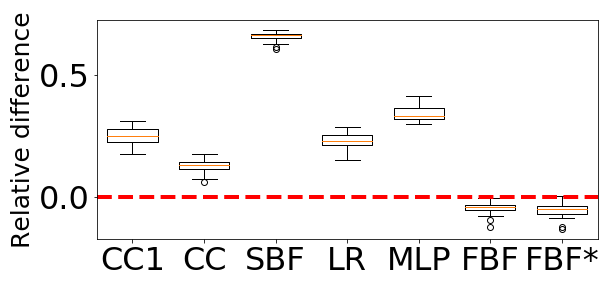}
    \caption{Syn. $\Real^{100}$}
    \label{fig:bcs-r100}
  \end{subfigure}
  ~
  \begin{subfigure}{0.23\textwidth}
    \centering
    \includegraphics[width=\textwidth]{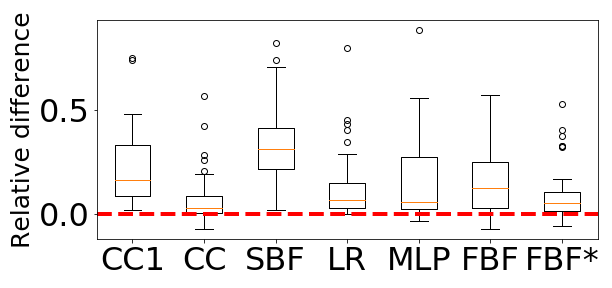}
    \caption{$d \in [10, 100]$.}
    \label{fig:bco-1}
  \end{subfigure}
  ~
  \begin{subfigure}{0.23\textwidth}
    \centering
    \includegraphics[width=\textwidth]{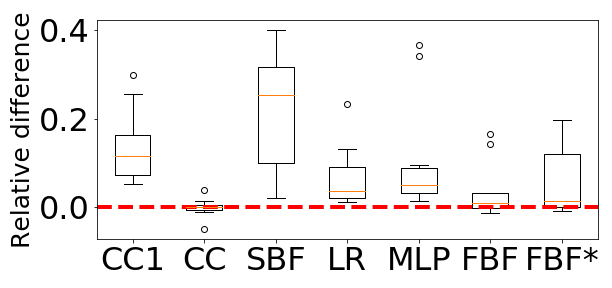}
    \caption{$d \in [101, 1024]$.}
    \label{fig:bco-2}
  \end{subfigure}
  \caption{Performance of \fbfc/$\fbfc^*$ and baselines relative to
    the \knnc performance on {\em synthetic} (\ref{fig:bcs-b100} \&
    \ref{fig:bcs-r100}) and {\em OpenML} data (\ref{fig:bco-1} \&
    \ref{fig:bco-2}). The $10$-fold cross-validated accuracy is
    considered for each of the data sets. The box-plots correspond to
    performance relative to \knnc ({\em lower is better}) aggregated
    over multiple data sets (see text for details). The red dashed
    line denotes \knnc performance.}
  \label{fig:bcomp-all}
\end{figure}

\subsubsection{Synthetic data}

We begin with binary synthetic data of the form considered in our
theoretical results -- points $x \in \HC^d$ with $|x| = b < d$. We
then consider synthetic data in $\Real^d$. We cover different values
of $d$ and $b$ and create a $5$-class classification sets with $3$
modes per class.  For each value of $d$ (and $b$), we create 30 data
sets with $1000$ points each. The aggregate performance of all
baselines (aggregated across all instantiations of $d = 100$ ($b =
20$)) is presented in Figures~\ref{fig:bcs-b100} and
\ref{fig:bcs-r100}. More results on synthetic data sets with different
values of $d$ (and $b$) are presented is Supplement S3.3.

The results indicate that \fbfc and $\fbfc^*$ are able to match \knnc
performance significantly better than all other single pass baselines.
The binary \fbfc matches the performance of $\fbfc^*$ in $\R^d$, but
lags behind on the lower dimensional binary sets.  As expected, \cc
performs significantly better than the other baselines on account of
being able to properly compress multi-modal classes, albeit requiring
multiple passes. \ccone performs significantly worse than \cc since
one cluster is not able to appropriately compress multi-modal classes
while maintaining the separation between the classes. \lr and \mlpc
perform similarly to \ccone. The proposed \fbfc and $\fbfc^*$
significantly outperform \sbfc, highlighting the need for sparse high
dimensional hashes to summarize multi-modal neighborhoods while
avoiding overlap between per-class {\fbf}s.
%
\subsubsection{OpenML data}
We consider classification (binary and multi-class) data sets from
OpenML with numerical columns. We utilize two groups of data sets of
following sizes: (i) $48$ data sets with $d \in [10, 100]$, $n \leq
50000$, and (ii) $10$ data sets with $d \in [101, 1024]$, $n \leq
10000$ (see precise details in Supplement S3.4).  We consider the same
procedure as above of tuning hyper-parameters for the $10$-fold
cross-validated accuracy for all baselines and the proposed scheme
relative to the best \knnc accuracy.  The results, aggregated across
all data sets in the two groups, are summarized in
Figures~\ref{fig:bco-1} and \ref{fig:bco-2}.
%

As with synthetic data, the results indicate that $\fbfc^*$ is able to
match the performance of \knnc for both $d \in [10, 100]$ and $d \in
[100, 1024]$ on a varied set of real data sets, with the binary \fbfc
falling behind on the lower dimensional sets. \fbfc has a median
relative performance of $0.12$ for $d \in [10, 100]$ compared to
$0.05$ for $\fbfc^*$, justifying the novel non-binary \fbf. The binary
\fbfc matches \knnc in higher dimensions -- both \fbfc and $\fbfc^*$
have a median relative performance of around $0.01$. \cc performs best
relative to \knnc overall. Both the proposed schemes are fairly
competitive with the multiple-pass \cc baseline while significantly
outperforming \ccone and \sbfc. \fbfc and $\fbfc^*$ are competitive to
\lr and \mlpc for the lower dimensional sets (relative performance of
$0.07$ and $0.06$ for \lr and \mlpc respectively) while edging ahead
in the higher dimensional sets (relative performance of $0.04$ and
$0.05$ for \lr and \mlpc respectively).
\subsubsection{Vision data}
\begin{wrapfigure}{R}{0.5\textwidth}
    \captionof{table}{Test accuracy (in \%) for vision sets.}
    \label{tab:bcomp-vision}
    \begin{small}
      \begin{sc}
        \begin{tabular}{lcccc}
          \toprule
          Method & Mnist & F-Mnist & Cifar10 & Cifar100 \\
          \midrule
          \knnc     & 97.36 & 85.90 & 31.65 & 14.38 \\
          \ccone    & 82.23 & 70.34 & 24.72 & 7.63  \\
          \cc       & 96.26 & 84.66 & 31.86 & 13.09 \\
          \sbfc     & 13.60 & 26.10 & 11.27 & 1.88  \\
          \lr       & 92.09 & 84.30 & 28.37 & 7.65  \\
          \mlpc     & 96.06 & 84.27 & 28.96 & 7.09  \\
          $\fbfc^*$ & 95.69 & 80.02 & 36.73 & 16.34 \\
          \bottomrule
        \end{tabular}
      \end{sc}
    \end{small}
\end{wrapfigure}
As a final comparison, we consider $4$ popular vision data
sets\footnotemark.
\footnotetext{See Table S1 in Supplement S3 for data details. Note
  that we are not claiming to be competitive with the state-of-the-art
  deep learning classifiers -- we are merely demonstrating the
  capability of our proposed scheme to be competitive to \knnc (and
  other single-pass baselines) on data sets from varied domains.}
In this experiment, we only consider $\fbfc^*$ (omitting \fbfc) and
tune hyper-parameters for all methods with a held-out set and report
the accuracy of the best hyper-parameters on the pre-defined test set
in Table~\ref{tab:bcomp-vision}.
The results indicate that $\fbfc^*$ is competitive to \cc for
\textsc{Mnist}, while outperforming all methods including \knnc
significantly on \textsc{Cifar10} \& \textsc{Cifar100}. With
\textsc{Fashion-Mnist}, \cc, \lr and \mlpc perform competitively to
\knnc while $\fbfc^*$ falls significantly behind.
$\fbfc^*$ significantly outperforms \ccone and \sbfc baselines as in
the previous comparisons.
\subsection{Scaling} \label{sec:emp-eval:scaling}
%
\begin{wrapfigure}{R}{0.5\textwidth}
  \centering
  \includegraphics[width=0.28\textwidth]{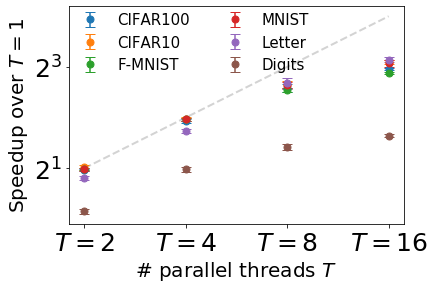}
  \caption{Scaling of parallelized \fbfc training with $T$ threads for
    $T = 1, 2, 4, 8, 16$. The \textcolor{gray}{gray} line corresponds
    to linear scaling. {\em Please view in color.}}
  \label{fig:scaling}
  \vskip -0.1in
\end{wrapfigure}
%
We evaluate the scaling of the parallelized \fbfc training
(Algorithm~\ref{alg:train-infer} (Train\fbfc)) with the number of
parallel threads. For fixed hyper-parameters, we average runtimes (and
speedups) over 10 repetitions for each of the $6$ data sets (see Table
S1 in Supplement S3) and present the results in
Figure~\ref{fig:scaling}.
The results indicate that the parallelized implementation of our
proposed scheme scales very well for up to $8$ threads for the larger
data sets. The parallelism shows significant gains (up to $2 \times$)
even for the tiny \textsc{Digits} data set, demonstrating the
parallelizability of the \fbfc training.
\subsection{Problem insights through class similarities}
\begin{figure}[htb]
  \vskip -0.1in
  \centering
  \begin{subfigure}{0.25\textwidth}
    \centering
    \frame{\includegraphics[width=0.8\textwidth]{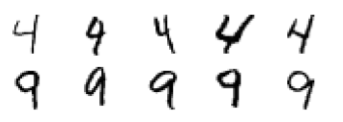}}
    \caption{MN \#1: 4 v 9}
    \label{fig:lsim:mnist1}
  \end{subfigure}
  ~
  \begin{subfigure}{0.25\textwidth}
    \centering
    \includegraphics[width=0.75\textwidth]{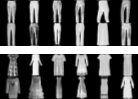}
    \caption{FM \#1: Trouser v Dress}
    \label{fig:lsim:fmnist1}
  \end{subfigure}
  ~
  \begin{subfigure}{0.35\textwidth}
    \centering
    \includegraphics[width=0.5\textwidth]{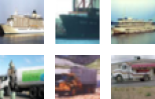}
    \caption{C10 \#1: Ship vs. Truck}
    \label{fig:lsim:cifar10-1}
  \end{subfigure}
  ~
  \begin{subfigure}{0.25\textwidth}
    \centering
    \frame{\includegraphics[width=0.8\textwidth]{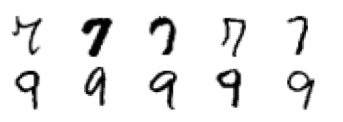}}
    \caption{MN \#2: 7 v 9}
    \label{fig:lsim:mnist2}
  \end{subfigure}
  ~
  \begin{subfigure}{0.25\textwidth}
    \centering
    \includegraphics[width=0.75\textwidth]{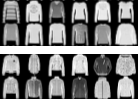}
    \caption{FM \#2: Pullover v Coat}
    \label{fig:lsim:fmnist2}
  \end{subfigure}
  ~
  \begin{subfigure}{0.35\textwidth}
    \centering
    \includegraphics[width=0.5\textwidth]{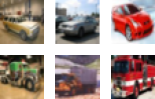}
    \caption{C10 \#2: Automobile vs. Truck}
    \label{fig:lsim:cifar10-2}
  \end{subfigure}
  \caption{Label pairs for MNIST (MN), Fashion-MNIST (FM) and CIFAR10
    (C10) with highest \fbf similarities.}
  \label{fig:label-sim}
  \vskip -0.1in
\end{figure}
%
We consider some of the vision data sets and explore the inter-class
similarities for the problems. For each data set, we report the top
$2$ most similar class pairs based on their respective trained \fbfc
in Figure~\ref{fig:label-sim}. For MNIST, the most similar pairs of
digits are $(4,9)$ and $(7,9)$. This is somewhat validated by the
images where these pairs are digits are visually hard to
distinguish. In Fashion-MNIST, the hard pairs are (trouser, dress) and
(pullover, coats). Trousers have the same long structure as dresses,
and pullovers have the same structure of a top with two long arm
sleeves. For CIFAR10, the most similar label pairs as per the \fbfc
class similarities are ``ship'' vs. ``truck'' and ``automobile''
vs. ``truck''. Both ship and truck images usually have pictures of
containers; trucks and automobiles are images of vehicles with
headlights, wheels and such.  The class similarities generated by
\fbfc seem reasonable for these data sets, implying that we can use
this scheme to estimate class similarities in other problems where the
class labels are not interpretable and there are no inter-class
hierarchies.


%
\section{Conclusions and future work} \label{sec:conc}
In this paper we proposed a novel neurosciene inspired Fly Bloom
Filter based classifier (\fbfc) that can be trained in an
\emph{embarrassingly parallelized} fashion in a single pass of the
training set -- a point never needs to be revisited, and the whole
training data does not need to be in memory. The inference requires an
efficient \fh followed by a very sparse dot product.  On the
theoretical side, we established conditions under which \fbfc agrees
with the nearest-neighbor classifier.  We empirically validated our
proposed scheme with over 50 data sets of varied data dimensionality
and demonstrated that the predictive performance of our proposed
classifier is competitive the the $k$-nearest-neighbor classifier and
other single-pass classifiers.

In the future we will pursue theoretical guarantees for \fbfc and
$\fbfc^*$ for general data in $\Real^d$ by exploring other data
dependent assumptions such as doubling measure. Utilizing the sparse
and randomized nature of \fbfc, we will also investigate differential
privacy preserving properties of \fbfc as well as robustness of \fbfc
to benign and adversarial perturbations.

\bibliography{references}
\bibliographystyle{unsrtnat}

\appendix
\renewcommand{\thesection}{S\arabic{section}}
\renewcommand{\thetable}{S\arabic{table}}
\renewcommand{\thefigure}{S\arabic{figure}}
\renewcommand{\theequation}{S\arabic{equation}}
\renewcommand{\thetheorem}{S\arabic{theorem}}

\setcounter{table}{0}
\setcounter{figure}{0}
\setcounter{equation}{0}
\setcounter{theorem}{0}
\setcounter{section}{0}

\section{Discussion on non-binary \fbfc} \label{asec:non-binary-fbfc}

Note that the $j^{\mbox{th}}$ coordinate of any $\fbf^*$ $w_i$ diminishes as the number of training examples $x'$ with label $i$ and nonzero $j^{\mbox{th}}$ coordinate in their \fh $h(x')$ increases.
In fact, we can control the number of data points that can affect the value of $w_{ij}$. To see this, choose any small $\epsilon>0$ such that if $w_{ij}\leq \epsilon$, then we can effectively assume $w_{ij}\approx 0$. Suppose $t=|\{(x',y')\in S ~:~ y'=i \mbox { and } (h(x'))_j=1\}|$. Then it is easy to see that,
$$
w_{ij}=(1-c)^t\leq e^{-ct}\leq\epsilon\Rightarrow t\geq\frac{1}{c}\ln(1/\epsilon)
$$
That means even if the set $|\{(x',y')\in S ~:~ y'=i \mbox { and } (h(x'))_j=1\}|$ may contain $t'>t$ data points, only $t$ of them control the value of $w_{ij}$. More importantly, (i) $t$ can be controlled by choosing appropriate $c$, and (ii) using the similarity preservation of the projection matrix $M_m^s$, any test point $x$ with $(h(x))_j=1$ will be close to those $t$ data point with high probability.
%
\section{Supplementary material from Section 4}

Stating \fh definition for completeness:

The basic building block of our proposed algorithm is a fruit-fly olfactory circuit inspired \fh function, first introduced by \citet{dasgupta2017neural}. For $x \in \Real^d$, the \fh function $h \colon \Real^d \to \HC^{m}$ is defined as,
\begin{equation} \label{aeq:flyhash}
    h(x) = \Gamma_\rho (M_{m}^s x),
\end{equation}
where $M_m^s \in \HC^{m \times d}$ is the randomized sparse lifting binary matrix with $s\ll d$ nonzero entries in each row, and $\Gamma_\rho \colon \Real^{m} \to \HC^{m}$ is the winner-take-all function converting a vector in $\Real^m$ to one in $\HC^m$ by setting the highest $\rho$ elements to $1$ and the rest to zero. For ease of notation, we use $M$ instead of $M_m^s$.

\subsection{Proof of Lemma 1}
Stating Lemma 1 for completeness:
\begin{lemma}\label{alem:expectation}
Fix any $x\in\R^d$ and let  $h(x)\in\{0,1\}^m$ be its \fh using equation~\ref{aeq:flyhash}. Let $x_{NN}^{i}=\argmin_{(x',y')\in S^{i}}\|x-x'\|$ for $i\in\{0,1\}$, 
where $\|\cdot\|$ is any distance metric. Let $A_{S^1}=\{\theta: \cap_{(x',y')\in S^{1}} ~ \theta^{\top} x'<\tau_{x'}(\rho/m)\}$ and $A_{S^0}=\{\theta: \cap_{(x',y')\in S^{0}} ~ \theta^{\top} x'<\tau_{x'}(\rho/m)\}$. Then the following holds, where the expectation is taken over the random choice of projection matrix $M$.\\
(i) $\E_M(\frac{w_{1}^{\top} h(x)}{\rho})=\pr_{\theta\sim Q}\left(A_{S^1} ~ | ~ \theta^{\top}x\geq \tau_x(\rho/m\right)$\\
(ii) $\E_M(\frac{w_{0}^{\top}h(x)}{\rho})=\pr_{\theta\sim Q}\left(A_{S^0} ~ | ~ \theta^{\top}x\geq \tau_x(\rho/m)\right)$\\
(iii) $\E_M(\frac{w_{1}^{\top} h(x)}{\rho})\geq 1-\sum_{x'\in S^{1}}q(x,x')$\\
(iv) $\E_M(\frac{w_{1}^{\top} h(x)}{\rho})\leq 1-q(x,x_{NN}^1)$\\
(v) $\E_M(\frac{w_{0}^{\top} h(x)}{\rho})\geq 1-\sum_{x'\in S^{0}}q(x,x')$\\
(vi) $\E_M(\frac{w_{0}^{\top} h(x)}{\rho})\leq 1-q(x,x_{NN}^0)$
\end{lemma}

\begin{proof}
Part (i) and (ii) follows from simple application of Lemma 2 of~\cite{dasgupta2018neural} to class specific {\fbf}s. Part (iii) and (v) follows from simple application of Lemma 3 of~\cite{dasgupta2018neural} to class specific {\fbf}s. For part (iv), simple application of Lemma 3 of~\cite{dasgupta2018neural} to \fbf $w_{1}$ ensures  that for any $x'\in S^{1}, \E_M(\frac{w_{1}^{\top} h(x)}{\rho})\leq 1-q(x,x')$. Clearly, $\E_M(\frac{w_{1}^{\top} h(x)}{\rho})\leq 1-q(x,x_{NN}^1)$. Applying similar argument, part (vi) also holds.
\end{proof}

\subsection{Proof of Theorem 3}
We analyze classification performance of \fbfc trained on a training set $S=\{(x_i,y_i)\}_{i=1}^{n_0+n_1}\subset \mathcal{X} \times  \{0,1\}$, where  $S=S^{1}\cup S^{0}$, $S^{0}\subset S$ with label 0 and $S^{1}\subset S$ with label 1, satisfying $|S^{0}|=n_0$ and $|S^{1}|=n_1$ and $n=\max\{n_0,n_1\}$. For appropriate choice of $m$, let $w_{0}, w_{1}\in\{0,1\}^m$ be the $\fbf$s constructed using $S^{0}$ and $S_1$ respectively. In Theorem 3, we consider a special case where examples from each class have binary feature vectors with fixed number of ones. In particular, $\mathcal{X}=\mathcal{X}_b=\{x\in\HC^d : |x|_1 = b < d\}$.

Restating Theorem 3 for completeness:

\begin{theorem}\label{ath:binary_new}
Let $S$ be a training set as given above. Fix any $\delta\in(0,1)$, and set $\rho\geq\frac{12}{\mu}\ln(4/\delta)$, $m\geq (d/b)n\rho$, and $s=\log_{d/b}(m/\rho)$, where $\mu=\min\left\{\E_M\left(\frac{w_0^{\top} h(x)}{\rho}\right),\E_M\left(\frac{w_1^{\top} h(x)}{\rho}\right)\right\}$ and $h(x)$ is the \fh function from equation~\ref{aeq:flyhash}. For any test example $x\in\mathcal{X}$, let its closest point from $S$ measured using $\ell_1$ metric be $x_{NN}$, having label $y_{NN}\in\{0,1\}$, satisfies, (i)  $\|x-x_{NN}\|_1\leq\frac{2b(1-b/d)}{3s}$, and  (ii) $\|x-x_i\|_1\geq 2b(1-b/d)$ for all $(x_i,y_i)\in S$, with $y_i\neq y_{NN}$. Let  $w_{0}, w_{1}\in\{0,1\}^m$ be the {\fbf}s constructed using $S^{0}$ and $S^{1}$ respectively. Then, with probability at least $1-\delta$ (over the random choice of projection matrix $M$),  prediction of \fbfc on $x$ agrees with the prediction of 1-NN classifier on $x$.
\end{theorem}


\begin{proof}
We first show that a result similar to the one we wish to prove holds in expectation (for exact statement, please  see Lemma~\ref{lem:binary_new} below). Using this result and standard concentration results presented in lemma~\ref{lem:concentration}, we show that the desired result holds with high probability, provided $\rho$ is large.

Using Lemma~\ref{lem:binary_new}, we show that $\E_M\left(\frac{w_1^{\top}h(x)}{\rho}\right)\leq s\epsilon$ and $\E_M\left(\frac{w_0^{\top}h(x)}{\rho}\right)\geq 1-\frac{b}{d}$. Therefore, if $\epsilon$ is restricted in the range $\left(0,\frac{(1-b/d)}{(\log_{d/b} (m/\rho))}\right)$, then $\E_M\left(\frac{w_1^{\top}h(x)}{\rho}\right)< \E_M\left(\frac{w_0^{\top}h(x)}{\rho}\right)$ which ensures that prediction of \fbfc on $x$ agrees with prediction of 1-NN classifier on $x$ in expectation. Now, using Lemma~\ref{lem:concentration}, with probability at least $1-\delta$, we have, $\frac{w_1^{\top} h(x)}{\rho}\leq \frac{3}{2}\E_M\left(\frac{w_1^{\top} h(x)}{\rho}\right)\leq \frac{3s\epsilon}{2}$ and $\frac{w_0^{\top} h(x)}{\rho}\geq \frac{1}{2}\E_M\left(\frac{w_0^{\top} h(x)}{\rho}\right)\geq \frac{1}{2}\left(1-\frac{b}{d}\right)$. Restricting $\epsilon$ in the range $\left(0,\frac{(1-b/d)}{3\log_{d/b} (m/\rho)}\right)$, ensures that $\|x-x_{NN}\|_1=2b\epsilon\leq \frac{2(1-b/d)}{3\log_{d/b} (m/\rho)}$, and with probability at least $1-\delta$,  $\frac{w_0^{\top} h(x)}{\rho} < \frac{w_1^{\top} h(x)}{\rho}$. The result follows.
\end{proof}

\begin{lemma}\label{lem:binary_new}
Let $S$ be a training set as given above.
For any test example $x\in\mathcal{X}$, let its closest point from $S$ measured using $\ell_1$ metric be $x_{NN}$ having label $y_{NN}\in\{0,1\}$. Assume that for all $(x_{i},y_{i})\in S$, with $y_{i}\neq y_{NN}$, $\|x-x_{i}\|_1\geq 2b(1-b/d)$ and $x_{NN}$ satisfies $\|x-x_{NN}\|_1\leq \frac{2b(1-b)}{\log_{d/b} (m/\rho)}$, where $m\geq (d/b)n\rho$. Let $s=\log_{d/b} (m/\rho)$ and  $w_{0}, w_{1}\in\{0,1\}^m$ be the {\fbf}s constructed using $S^{0}$ and $S^{1}$ respectively. Then, in expectation (over the random choice of projection matrix $M$),  prediction of \fbfc on $x$ agrees with the prediction of 1-NN classifier on $x$.
\end{lemma}

\begin{proof}
Without loss of generality, assume assume that $x_{NN}$ satisfies the relation $\|x- x_{NN}\|_1 =2b\epsilon$ for some $0<\epsilon<1$ and $y_{NN}=1$. Clearly, 1-NN classifier will predict $x$'s class label to be 1.

Let $h(x)\in\{0,1\}^m$ be the \fh function from equation~\ref{aeq:flyhash}. To ensure that prediction of \fbfc on $x$ agrees with that of 1-NN classifier on expectation, we need to show that $\E_M\left(\frac{w_1^{\top} h(x)}{\rho}\right)< \E_M\left(\frac{w_0^{\top} h(x)}{\rho}\right)$. Our plan is to show that upper bound of $\E_M\left(\frac{w_1^{\top} h(x)}{\rho}\right)$ is strictly smaller then lower bound of $\E_M\left(\frac{w_0^{\top} h(x)}{\rho}\right)$. Towards this end, for any $x\in\mathcal{X}_b$, set the threshold $\tau_x(k/m)$ to be $s$, whose value will be chosen later. Then we have,
\begin{eqnarray*}
\pr_{\theta\sim Q}(\theta\cdot x\geq\tau_x(\rho/m))&=&\pr_{\theta\sim Q}(\theta\cdot x\geq s)\\
&=&\pr_{\theta\sim Q}(\theta\cdot x=s)\\
&=&\frac{\binom{b}{s}}{\binom{d}{s}}\approx \left(\frac{b}{d}\right)^{s}
\end{eqnarray*}
where the second inequality follows from the fact that $\theta$ has exactly $s$ ones and maximum value of $\theta^{\top} x$ is $s$.
Since $\pr_{\theta\sim Q)}(\theta^{\top} x\geq\tau_x(\rho/m))=\rho/m$, we have $s\approx \frac{\log(m/\rho)}{\log(d/b)}=\log_{d/b}(m/\rho)$.  Additionally, from Lemma 6 of~\cite{dasgupta2018neural} we have ,
\begin{equation}\label{eq:binary_q}
    q(x,x')\approx\left(\frac{x^{\top} x'}{b}\right)^{s}
\end{equation}
This approximation is excellent when $c$ is small relative to $x\cdot x'$. We will henceforth take it to be equality.
 It is easy to check that for any $x,x'\in\mathcal{X}_b, \|x-x'\|_1=2(b-x\cdot x')$. Therefore, $\|x-x_{NN}\|=2b\epsilon$ implies $x^{\top} x_{NN}=b(1-\epsilon)$ and for all $(x',y')\in S^0, \|x-x'\|_1\geq2b(1-b/d)$ implies $x^{\top} x'\leq (b/d)b$. Therefore, using equation~\ref{eq:binary_q}, we have $q(x,x_{NN})=\left(\frac{x^{\top} x_{NN}}{b}\right)^{s}= (1-\epsilon)^{s}\geq 1-s\epsilon$. Combining this with part (iv) of Lemma~\ref{alem:expectation}, we have $\E_M\left(\frac{w_{1}^{\top} h(x)}{\rho}\right)\leq 1-q(x,x_{NN})\leq 1-(1-s\epsilon)=s\epsilon$. Since for each $(x',y')\in S^0$, $x^{\top} x'\leq (b/d)b$, we have $q(x,x')=\left(\frac{x\cdot x'}{b}\right)^{s}\leq \left(\frac{b}{d}\right)^{s}=\rho/m$. Combining this with part (v) of Lemma~\ref{alem:expectation}, we have $\E_M\left(\frac{w_{0}^{\top} h(x)}{\rho}\right)\geq 1-\sum_{(x',y')\in S^0}q(x,x')\geq 1-\frac{n_0\rho}{m}\geq 1-b/d$. To ensure that the lower bound of $\E_M\left(\frac{w_{0}^{\top} h(x)}{\rho}\right)$ is strictly larger than upper bound of  $\E_M\left(\frac{w_{1}^{\top} h(x)}{\rho}\right)$, we need, $s\epsilon<(1-b/d)\Rightarrow \epsilon<\frac{(1-b/d)}{s}=\frac{(1-b/d)}{\log_{d/b} (m/\rho)}$, which ensures $\|x-x_{NN}\|_1=2b\epsilon\leq\frac{2b(1-b/d)}{\log_{d/b} (m/\rho)}$.

 Since for any test data point $x$, its closet point in $S$ can also have label 0, we simply replace $n_0$ by $n=\max\{n_0,n_1\}$.
\end{proof}

\subsection{Auxiliary Lemma and its proof}
The following concentration result is standard and a similar form has appeared in~\cite{dasgupta2018neural}.
\begin{lemma}\label{lem:concentration}
Let $x_{1},\ldots, x_{n_1}\in\mathcal{X}_b$ be the unlabeled examples of $S^1$ and let $\tilde{x}_1,\ldots,\tilde{x}_{n_0}\in\mathcal{X}_b$ be the unlabeled examples of $S^0$ from Lemma~\ref{lem:binary_new}. Pick any $\delta\in(0,1)$ and $x\in \mathcal{X}_b$. With probability at least $1-\delta$ over the choice of random projection matrix $M$, the following holds,\\
(i) $\frac{1}{2}\E_M\left(\frac{w_1^{\top} h(x)}{\rho}\right)\leq \frac{w_1^{\top} h(x)}{\rho}\leq \frac{3}{2}\E_M\left(\frac{w_1^{\top} h(x)}{\rho}\right)$ \\
(ii) $\frac{1}{2}\E_M\left(\frac{w_0^{\top} h(x)}{\rho}\right)\leq \frac{w_0^{\top} h(x)}{\rho}\leq \frac{3}{2}\E_M\left(\frac{w_0^{\top} h(x)}{\rho}\right)$\\
provided $\rho\cdot\min\left\{\E_M\left(\frac{w_0^{\top} h(x)}{\rho}\right),\E_M\left(\frac{w_1^{\top} h(x)}{\rho}\right)\right\}\geq 12\ln(4/\delta)$.
\end{lemma}
\begin{proof}
We will only prove part (i) since part (ii) is similar. Let $h(x), h(x_1),\ldots,h(x_{n_1})$ be the projected-and-thresholded versions of $x,x_1,\ldots, x_{n_1}$ respectively. Define random variables $U_1,\ldots, U_m\in \{0,1\}$ as follows:
\[
    U_j=
\begin{cases}
    1,& \text{if } h(x_1)_j=\cdots=h(x_{n_1})_j=0 \text{ and } h(x)_j=1\\
    0,              & \text{otherwise}
\end{cases}
\]
The $U_j$ are i.i.d. and
\begin{eqnarray*}
\E_M(U_j)\hspace{-0.12in}&=&\hspace{-0.12in}\pr_M(h(x)_j=1)\times\\
&&\hspace{-0.12in}\pr_M\left(h(x_1)_j=\cdots=h(x_{n_1})_j=0 ~ | ~ h(x)_j=1\right)\\
&=&\hspace{-0.12in}\frac{\rho}{m}\E_M\left(\frac{w_1^{\top} h(x)}{\rho}\right)
\end{eqnarray*}
where we have used the fact that $\pr_M(h(x)_j=1)=\pr_{\theta\sim Q}(\theta^{\top} x\geq \tau_x(\rho/m))=\rho/m$ and using Lemma 2 of the supplementary material of~\cite{dasgupta2018neural}, $\pr_M\left(h(x_1)_j=\cdots=h(x_{n_1})_j=0 ~ | ~ h(x)_j=1\right)=\E_M\left(\frac{w_1^{\top} h(x)}{\rho}\right)$. Therefore, $\E_M(U_1+\cdots+U_m)=\rho\cdot\E_M\left(\frac{w_1^{\top} h(x)}{\rho}\right)$. Let $\mu_1=\E_M\left(\frac{w_1^{\top} h(x)}{\rho}\right)$. By multiplicative Chernoff bound for any $0<\epsilon<1$, we have,
$$
\pr_M\left(U_1+\cdots+U_m\geq (1+\epsilon)\rho\mu_1\right)\leq\exp(-\epsilon^2\rho\mu_1/3)
$$
$$
\pr_M\left(U_1+\cdots+U_m\leq (1-\epsilon)\rho\mu_1\right)\leq\exp(-\epsilon^2\rho\mu_1/2)
$$
Setting $\epsilon=1/2$ and bounding right hand side of each of the above two inequalities by $\delta/4$,  ensures that part (i) holds with probability at least $1-\frac{\delta}{2}$ provided $\rho\cdot\E_M\left(\frac{w_1^{\top} h(x)}{\rho}\right)\geq 12\ln(4/\delta)$.
\end{proof}
\subsection{Result for multi-class classification}\label{sec:sm-multi-class}

Theorem \ref{ath:binary_new} can be easily extended to multi-class classification problem involving $L$ classes in a straight forward manner by applying concentration result to each of the $\left(\frac{w_i^{\top}h(x)}{\rho}\right)$, for $i\in [L]$, and using a union bound.

\begin{cor}\label{cor:multi}
Given a training set $S=\{(x_i,y_i)\}_{i=1}^{\sum_{j=1}^{L-1}n_j}\subset \mathcal{X}_b \times \mathcal{Y}\subset \{0,1\}^d\times \{0,1,\ldots,L-1\}$ of size $\sum_{i=0}^{L-1}n_i$, let $S=\cup_{i=0}^{L-1} S^i$, where $S^{i}$ is the subset of $S$ with label $i$ satisfying $|S^i|=n_i$ and $n=\max\{n_0,\ldots,n_{L-1}\}$. For any test example $x\in\mathcal{X}_b$, let its closest point from $S$ measured using $\ell_1$ metric be $x_{NN}$ having label $y_{NN}\in\{0,\ldots,L-1\}$. Fix any $\delta\in(0,1)$ and set $\rho\geq\frac{12}{\mu}\ln(2L/\delta)$, 
 $m\geq (d/b)n\rho$, and $s=\log_{d/b}(m/\rho)$, where  $\mu=\min\left\{\E_M\left(\frac{w_0^{\top} h(x)}{\rho}\right),\ldots,\E_M\left(\frac{w_{L-1}^{\top} h(x)}{\rho}\right)\right\}$ and $h(x)$ is the \fh function from equation~\ref{aeq:flyhash}. Assume that for all $(x_i,y_i)\in S$, with $y_i\neq y_{NN}, \|x-x_i\|_1\geq 2b(1-b/d)$ and $x_{NN}$ satisfies  $\|x-x_{NN}\|_1\leq\frac{2b(1-b/d)}{3s}$. Let $w_{0},\ldots, w_{L-1}\in\{0,1\}^m$ be the {\fbf}s constructed using $S^{0},\ldots, S^{L-1}$ respectively. Then, with probability at least $1-\delta$ (over the random choice of projection matrix $M$),  prediction of \fbfc on $x$ agrees with the prediction of 1-NN classifier on $x$.
\end{cor}

\subsection{Proof of Theorem 4}

We analyze classification performance of \fbfc trained on a training set $S=\{(x_i,y_i)\}_{i=1}^{n_0+n_1}\subset \mathcal{X} \times  \{0,1\}$, where  $S=S^{1}\cup S^{0}$, $S^{0}\subset S$ with label 0 and $S^{1}\subset S$ with label 1, satisfying $|S^{0}|=n_0$ and $|S^{1}|=n_1$ and $n=\max\{n_0,n_1\}$. For appropriate choice of $m$, let $w_{0}, w_{1}\in\{0,1\}^m$ be the $\fbf$s constructed using $S^{0}$ and $S_1$ respectively. In Theorem 3, we consider a special case where we make permutation invariant distribution assumption. Permutation invariant distribution in the \fbf context was first introduced in \cite{dasgupta2018neural} and is defined as follows: a distribution $P$ over $\Real^d$ is permutation invariant if for any permutation $\sigma$ of $\{1,2,\ldots,d\}$ and any $x=(x_1,\ldots,x_d)\in\Real^d$, $P(x_1,\ldots,x_d)=P(x_{\sigma(1)},\ldots,x_{\sigma(d)})$ . Restating Theorem 4 for completeness.

\begin{theorem}\label{ath:real_new}

Let $S$ be a training set as given above. Fix any $\delta\in(0,1)$, $s\ll d$, and set $\rho\geq\frac{48}{\mu}\ln(8/\delta)$ and $m\geq 14n\rho/\delta$, where $\mu=\min\left\{\E_M\left(\frac{w_0^{\top} h(x)}{\rho}\right),\E_M\left(\frac{w_1^{\top} h(x)}{\rho}\right)\right\}$, $h(x)$ is the \fh function from equation~\ref{aeq:flyhash}, and  $w_{0}, w_{1}\in\{0,1\}^m$ are the {\fbf}s constructed using $S^{0}$ and $S^{1}$ respectively. For any test example $x\in\Real^d$, sampled from a permutation invariant distribution, let $x_{NN}$ be its nearest neighbor from $S$ measured using $\ell_{\infty}$ metric, which  satisfies $\|x-x_{NN}\|_{\infty}\leq \Delta/s$, where $\Delta=\frac{1}{2}\left(\tau_x(2\rho/m)-\tau_x(\rho/m)\right)$ and has label $y_{NN}\in\{0,1\}$. Then, with probability at least $1-\delta$ (over the random choice of projection matrix $M$),  prediction of \fbfc on $x$ agrees with the prediction of 1-NN classifier on $x$.
\end{theorem}

\begin{proof}
Without loss of generality, assume that $y_{NN}=1$. For the case when $y_{NN}=0$, is similar. Prediction of \fbfc on $x$ agrees with the prediction of 1-NN classifier whenever $\left(w_1^{\top}h(x)/\rho\right)<\left(w_0^{\top}h(x)/\rho\right)$. We first show that $\E_M\left(w_1^{\top}h(x)/\rho\right)<\E_M\left(w_0^{\top}h(x)/\rho\right)$ with high probability and then using standard concentration bound presented in lemma \ref{lem:concentration}, we achieve the desired result. Since $\|x-x_{NN}\|_{\infty}\leq \Delta/s$, using lemma 9 of \cite{dasgupta2018neural}, we get $q(x,x_{NN})\geq 1/2$. Combining this with part (iv) of lemma \ref{alem:expectation}, we get $\E_M\left(w_1^{\top}h(x)/\rho\right)\leq 1/2$. Next, since $x$ is sampled from a permutation invariant distribution, using corollary 11 of \cite{dasgupta2018neural}, we get $\E_x q(x,x_i)=\rho/m$ for each $x'\in S^0$, and thus using linearity of expectation, $\E_x\left(\sum_{x'\in S^0}q(x,x')\right)=\sum_{x'\in S^0}\E_x q(x,x')=\rho n_0/m$. For any $\alpha>0$,  using Markov's inequality, 
$$\pr\left(\sum_{x'\in S^0}q(x,x')>\alpha\right)\leq\frac{\E_x\left(\sum_{x'\in S^0}q(x,x'\right)}{\alpha}=\frac{\rho n_0}{m\alpha}\leq \frac{\delta}{2}.$$
Therefore, $\sum_{x'\in S^0}q(x,x')\leq \alpha$ with probability at least $1-\delta/2$ for $m\geq\frac{2\rho n_0}{\alpha\delta}$. Combining this with part (v) of lemma \ref{alem:expectation}, we immediately get, $\E_M\left(w_0^{\top}h(x)/\rho\right)\geq 1-\alpha$ with probability at least $1-\delta/2$. It is easy to see that for $\alpha<1/2$, $\E_M\left(w_1^{\top}h(x)/\rho\right)<\E_M\left(w_0^{\top}h(x)/\rho\right)$ with  probability at least $1-\delta/2$, and thus in expectation, prediction of \fbfc on $x$ agrees with the prediction of 1-NN classifier on $x$. Using concentration bound and a smaller $\alpha$, we next show that $\left(w_1^{\top}h(x)/\rho\right)<\left(w_0^{\top}h(x)/\rho\right)$ with probability at least $1-\delta$. In particular, using $\epsilon=1/4$ and $\delta=\delta/2$ in lemma \ref{lem:concentration}, we see that with probability at least $1-\delta/2$ the following holds: (i) $\frac{3}{4}\E_M\left(w_1^{\top}h(x)/\rho\right)\leq w_1^{\top}h(x)/\rho\leq \frac{5}{4}\E_M\left(w_1^{\top}h(x)/\rho\right)$, and (ii) $\frac{3}{4}\E_M\left(w_0^{\top}h(x)/\rho\right)\leq w_0^{\top}h(x)/\rho\leq \frac{5}{4}\E_M\left(w_0^{\top}h(x)/\rho\right)$ provided  $\rho\cdot\min\left\{\E_M\left(w_0^{\top}h(x)/\rho\right),\E_M\left(w_1^{\top}h(x)/\rho\right)\right\}\geq 48\ln(8/\delta)$. Combining this with the bounds on the expected values of the novelty scores, it is easy to see that with probability $1-\delta$, $w_1^{\top}h(x)/\rho < w_0^{\top}h(x)/\rho$ whenever, $\frac{1}{2}\cdot\frac{5}{4} < (1-\alpha)\cdot\frac{3}{4}\Rightarrow \alpha <1/6$. Since $n=\max\{n_0,n_1\}\geq n_0$, setting $\alpha=1/7$, which in turn requires $m\geq 14n\rho/\delta$, the result follows.
\end{proof}
The above result can be extended to multi-class classification problem in a straight forward manner.

\section{Supplementary material from Section 5}
\label{asec:expts}

\paragraph{Implementation \& Compute Resource:} The proposed novel classification scheme is implemented in Python 3.6 to fit the \sklearn API~\citep{pedregosa2011scikit}, but the current implementation is not optimized for computational performance. We use the \sklearn implementation of various baselines we consider in our evaluations. To generate synthetic data sets, we use the \texttt{data.make\_classification} functionality in \sklearn~\citep{guyon2003design}. The experiments are performed on a 16-core 128GB machine running Ubuntu 18.04.

\begin{table}[!htb]
  \caption{Details of a subset of the data sets. For CIFAR-10 and CIFAR-100, we collapse the 3 color channels and then flatten the $32 \times 32$ images to points in $\Real^{1024}$. For MNIST and Fashion-MNIST, we flatten the $28 \times 28$ images to points in $\Real^{784}$.}
  \label{tab:data-stats}
  \vskip 0.15in
  \begin{center}
    \begin{small}
      \begin{sc}
        \begin{tabular}{lcccc}
          \toprule
          Data set & $n$ & $d$ & $L$ & Experiment\\
          \midrule
          Digits        & $1797$  & $64$   & $10$  & OpenML \\
          Letters       & $20000$ & $16$   & $26$  & OpenML \\
          Segment       & $2310$  & $19$   & $7$   & OpenML \\
          Gina Prior 2  & $3468$  & $784$  & $10$  & OpenML \\
          USPS          & $9294$  & $256$  & $10$  & OpenML \\
          Madeline      & $3140$  & $259$  & $2$   & OpenML \\
          \midrule
          MNIST         & $60000$ & $784$  & $10$  & Vision \\
          Fashion-MNIST & $60000$ & $784$  & $10$  & Vision \\
          CIFAR-10      & $50000$ & $1024$ & $10$  & Vision \\
          CIFAR-100     & $50000$ & $1024$ & $100$ & Vision \\
          \bottomrule
        \end{tabular}
      \end{sc}
    \end{small}
  \end{center}
  \vskip -0.1in
\end{table}

\subsection{Dependence on \fbfc hyper-parameters} \label{asec:expts:hpdep}

We study the effect of the different hyper-parameters of \fbfc -- (i) the dimensionality of the \fh $m$, (ii) the per-row density $s$ of the sparse binary projection matrix $M_m^s$, (iii) the NNZ $\rho$ in the \fh after the winner-take-all operation, and (iv) the decay rate $c$ of the \fbf. For this analysis, we consider $6$ data sets from OpenML -- Digits, Letters, Segment, Gina Prior 2, USPS and Madeline (see Table~\ref{tab:data-stats} for data sizes). For every hyper-parameter setting, we compute the $10$-fold cross-validated accuracy. We vary each hyper-parameter while fixing the others. The results for each of the hyper-parameters and data sets are presented in Figure~\ref{afig:hpdep3-1}~\&~\ref{afig:hpdep3-2}. We evaluate the following configurations for the evaluation of each of the hyper-parameters:
\begin{itemize}
    \item {\bf \fh dimension $m$:} We try $10$ values for $m \in [4d, 4096d]$ with $(s/d) \in \{ 0.1, 0.3 \}$, $\rho \in \{ 8, 32 \}$, $c \in \{0.5, 1\}$.
    \item {\bf Projection density $s/d$:} We try $10$ values for $(s/d) \in [0.1, 0.8]$ with $m \in \{ 256, 1024 \}$, $\rho \in \{ 8, 32 \}$, $c \in \{ 0.5, 1 \}$.
    \item {\bf \fh NNZ $\rho$:} We try $10$ values for $\rho \in [4, 256]$ with $m \in \{ 256, 1024 \}$, $(s/d) \in \{ 0.1, 0.3 \}$, $c \in \{ 0.5, 1 \}$.
    \item {\bf \fbf decay rate $c$:} We try $10$ values for $c \in [0.2, 0.9] $ and $c = 1$ with $m \in \{ 256, 1024 \}$, $(s/d) \in \{ 0.1, 0.3 \}$, $\rho \in \{ 8, 32 \}$.
\end{itemize}
\begin{figure*}[thb]
  \centering
    \begin{subfigure}{0.235\textwidth}
      \centering
      \includegraphics[width=\textwidth]{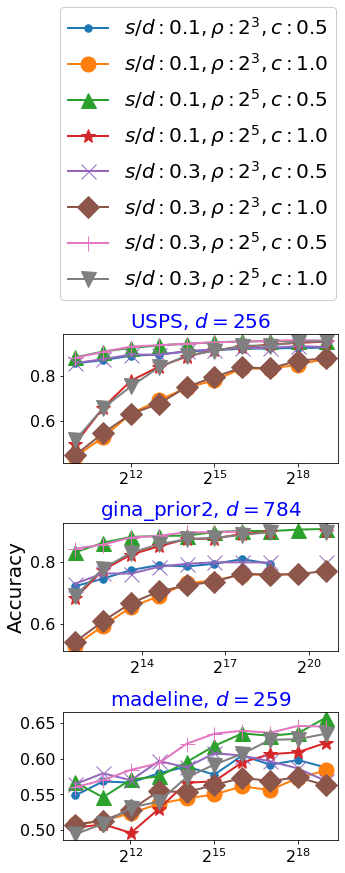}
      \caption{\fh dimension $m$}
      \label{figs:hpdep3-ef-1}
    \end{subfigure}
    ~
    \begin{subfigure}{0.235\textwidth}
      \centering
      \includegraphics[width=\textwidth]{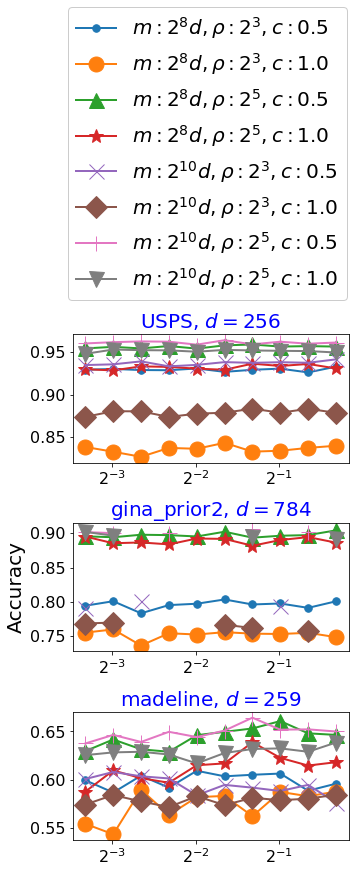}
      \caption{Projection density $s/d$}
      \label{figs:hpdep3-cs-1}
    \end{subfigure}
    ~
    \begin{subfigure}{0.235\textwidth}
      \centering
      \includegraphics[width=\textwidth]{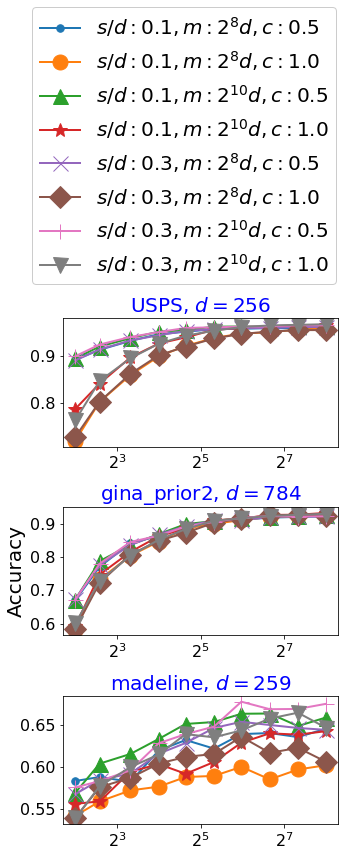}
      \caption{\fh NNZ $\rho$}
      \label{figs:hpdep3-wn-1}
    \end{subfigure}
    ~
    \begin{subfigure}{0.235\textwidth}
      \centering
      \includegraphics[width=\textwidth]{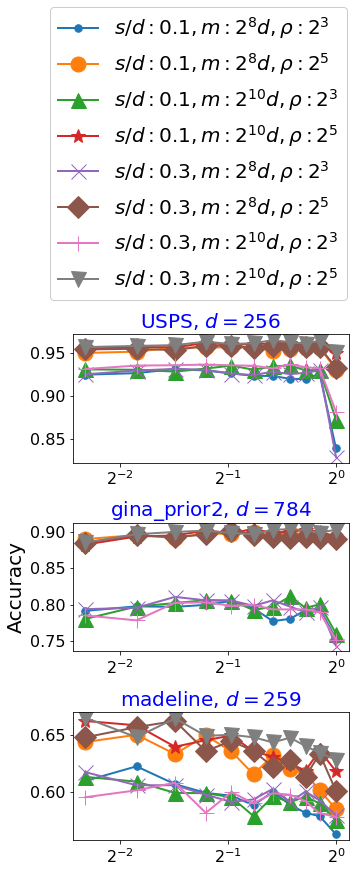}
      \caption{\fbf decay rate $c$}
      \label{figs:hpdep3-c-1}
    \end{subfigure}
    \caption{{\bf \fbfc hyper-parameter dependence -- Part I.} Effect of the different \fbfc hyper-parameters $m$, $s$, $\rho$, $c$ on \fbfc predictive performance for $3$ data sets -- the horizontal axes correspond to the hyper-parameter being varied while fixing the remaining hyper-parameters. The vertical axes correspond to the 10-fold cross-validated accuracy for the given hyper-parameter configuration ({\em higher is better}). Note the log scale on the horizontal axes. For the hyper-parameter $c$, $c=1$ corresponds to the binary \fbfc. {\em Please view in color.} }
    \label{afig:hpdep3-1}
\end{figure*}
\begin{figure*}[thb]
  \centering
    \begin{subfigure}{0.235\textwidth}
      \centering
      \includegraphics[width=\textwidth]{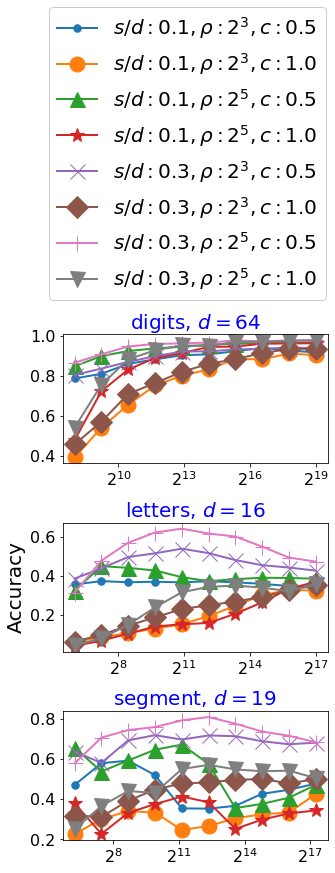}
      \caption{\fh dimension $m$}
      \label{figs:hpdep3-ef-2}
    \end{subfigure}
    ~
    \begin{subfigure}{0.235\textwidth}
      \centering
      \includegraphics[width=\textwidth]{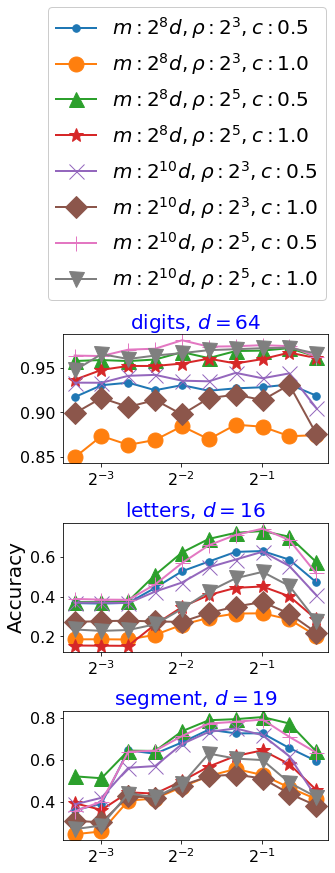}
      \caption{Projection density $s/d$}
      \label{figs:hpdep3-cs-2}
    \end{subfigure}
    ~
    \begin{subfigure}{0.235\textwidth}
      \centering
      \includegraphics[width=\textwidth]{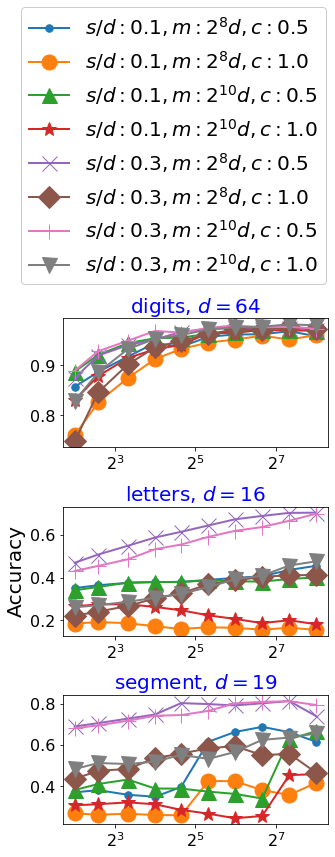}
      \caption{\fh NNZ $\rho$}
      \label{figs:hpdep3-wn-2}
    \end{subfigure}
    ~
    \begin{subfigure}{0.235\textwidth}
      \centering
      \includegraphics[width=\textwidth]{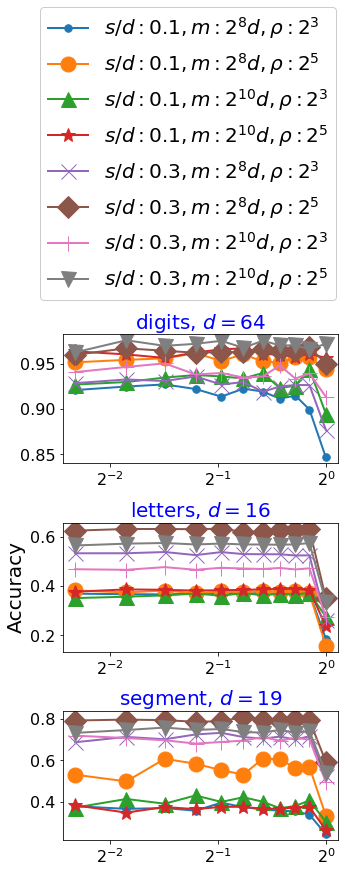}
      \caption{\fbf decay rate $c$}
      \label{figs:hpdep3-c-2}
    \end{subfigure}
    \caption{{\bf \fbfc hyper-parameter dependence -- Part II.} Effect of the different \fbfc hyper-parameters $m$, $s$, $\rho$, $c$ on \fbfc predictive performance for $3$ data sets -- the horizontal axes correspond to the hyper-parameter being varied while fixing the remaining hyper-parameters. The vertical axes correspond to the 10-fold cross-validated accuracy for the given hyper-parameter configuration ({\em higher is better}). Note the log scale on the horizontal axes. For the hyper-parameter $c$, $c=1$ corresponds to the binary \fbfc. {\em Please view in color.} }
    \label{afig:hpdep3-2}
\end{figure*}
The results in Figures~\ref{figs:hpdep3-ef-1}~\&~\ref{figs:hpdep3-ef-2} indicate that, for fixed $\rho$ increasing $m$ improves the \fbfc accuracy, aligning with the theoretical guarantees, up until an upper bound. This behavior is clear for high dimensional data sets. This behavior is a bit more erratic for the lower dimensional sets. Larger values of $m$ improve performance, since it allows us to capture each class' distribution with smaller random overlap between each class' {\fbf}s. But the theoretical guarantees also indicate that $\rho$ needs to be large enough, and if $m$ grows too large for any given $k$, the \fbfc accuracy might not improve any further.

Figures~\ref{figs:hpdep3-cs-1}~\&~\ref{figs:hpdep3-cs-2} indicate that for lower dimensional data (such as $d \leq 20$), increasing the projection density $s$ improves performance up to a point (around $s = 0.5$), after which the performance starts degrading. This is probably because for smaller values of $s$, not enough information is captured by the sparse projection for small $d$; for large values of $s$, each row in the projection matrix $M_m^s$ become similar to each other, hurting the similarity-preserving property of \fh. For higher dimensional data sets, the \fbfc performance appears to be somewhat agnostic to $s$ for any fixed $m$, $\rho$ and $c$.

Figures~\ref{figs:hpdep3-wn-1}~\&~\ref{figs:hpdep3-wn-2} indicate that increase in $\rho$ leads to improvement in \fbfc performance since large values of $\rho$ better preserve pairwise similarities. However, if $\rho$ is too large relative to $m$, the sparsity of the subsequent per-class \fbf go down, thereby leading to more overlap in the per-class {\fbf}s. So $\rho$ needs to large as per the theoretical analysis, but not too large.

Figures~\ref{figs:hpdep3-c-1}~\&~\ref{figs:hpdep3-c-2} indicate that the \fbfc is somewhat agnostic to the \fbf decay rate $c$ for any value strictly less that $1$ (corresponding to the binary \fbf). But there is a significant drop in the \fbfc performance from $c < 1$ to $c = 1$ across all data set -- this behavior is fairly consistent and apparent.
\subsection{Details on baselines} \label{asec:expts:baselines}

Here we detail all the baselines considered in our empirical evaluations and their respective hyper-parameter and the subsequent hyper-parameter optimization.

\begin{enumerate}[noitemsep]
  \item {\bf \knnc.} We consider the \knnc as the primary baseline to match where we tune over the size of the neighborhood in the range $[1,64]$ to maximize the $10$-fold cross-validated accuracy for each data set (synthetic or real).
\item {\bf \ccone.} Classification based on a single prototype per class, where the prototype of a class is the geometric center of the class, which can be computed with a single pass of the data.
\item {\bf \sbfc.} Classification via a variation of \fbfc where we utilize \sh/SRP~\citep{charikar2002similarity} instead of \fh to give us the \sh Bloom Filter classifier (\sbfc). We consider this baseline to demonstrate the need of the highly sparse hashes generated by \fh -- the hashes from \sh are not explicitly designed to be sparse. The dimensionality of the \sh $m$ is the hyper-parameter we search over --  we consider both projecting down in the range $m \in [1, d]$ (the traditional use) and projecting up $m \in [d, 2048d]$, where $d$ is the data dimensionality. Note that for the same projected dimension $m$, \sh is more expensive that \fh since \sh involves a dense matrix-vector multiplication instead of the sparse matrix-vector in \fh.
\item {\bf \lr.} We consider logistic regression trained for a single epoch with a stochastic algorithm. We utilize the \sklearn implementation (\texttt{linear\_model.LogisticRegression}) and tune over the following hyper-parameters -- (a) penalty type ($\ell_1$/$\ell_2$), (b) regularization $\in \left[ 2^{-10}, 2^{10} \right]$, (c) choice of solver (liblinear~\citep{fan2008liblinear}/SAG~\citep{schmidt2017minimizing}/SAGA~\citep{defazio2014saga}), (d) with/without intercept, (e) one-vs-rest or multinomial for multi-class, (f) with/without class balancing (note that this class balancing operation makes this a two-pass algorithm since we need the first pass to weigh the classes appropriately). We consider a total of 960 hyper-parameter configurations for each experiment.
\item {\bf \mlpc.} We consider a multi-layer perceptron trained for a single epoch with the ``Adam'' stochastic optimization scheme~\citep{kingma2014adam}. We use \texttt{sklearn.neural\_network.MLPClassifier} and tune over the following hyper-parameters -- (a) number of hidden layers $\{1, 2\}$, (b) number of nodes in each hidden layer $\{16, 64, 128\}$, (b) choice of activation function (ReLU/HyperTangent), (d) regularization, (e) batch size $ \in \left[2, 2^8\right]$, (f) initial learning rate $\in \left[10^{-5}, 0.1\right]$ (the rest of the hyper-parameters are left as \sklearn defaults). This leads to a total of 720 hyper-parameters configurations per experiment.
\item {\bf \cc.} We also consider a generalization of \ccone where we classify based on multiple prototypes per class -- a test point is assigned the label of its closest prototype. We generate the prototypes per class by $k$-means clustering (with multiple restarts) and tune over the choice of number of clusters per class in the range $[1, 64]$. This is {\em not a single pass baseline} but we consider this as a baseline since it is a common compression technique for {\knnc}.
\end{enumerate}
\subsection{Additional evaluations on synthetic data} \label{asec:expts:syn}

Here we present the relative performance of \fbfc and $\fbfc^*$ for different data dimensionalities in Figure~\ref{afig:bcomp-syn}.

\begin{figure}[htb]
  \centering
  \begin{subfigure}{0.23\textwidth}
    \includegraphics[width=\textwidth]{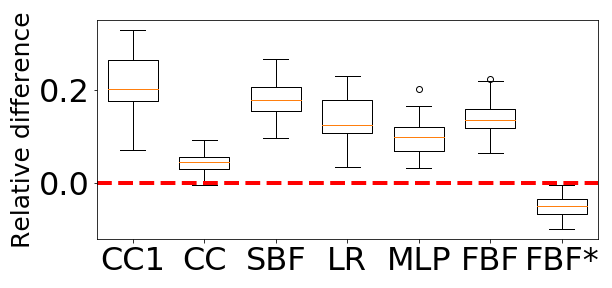}
    \caption{$\HC^{50}, b = 20$}
    \label{afig:bcs-b50}
  \end{subfigure}
  ~
  \begin{subfigure}{0.23\textwidth}
    \includegraphics[width=\textwidth]{figs/syn/{All.Synthetic.binary.d100.S1000}.png}
    \caption{$\HC^{100}, b = 40$}
    \label{afig:bcs-b100}
  \end{subfigure}
  ~
  \begin{subfigure}{0.23\textwidth}
    \includegraphics[width=\textwidth]{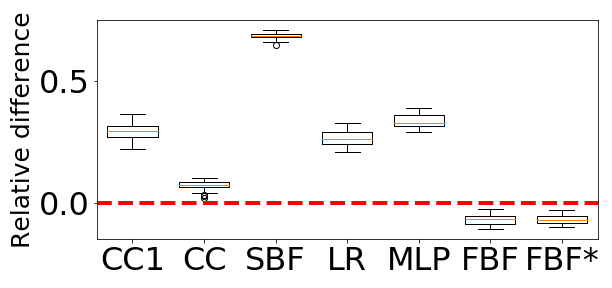}
    \caption{$\Real^{50}$}
    \label{afig:bcs-r30}
  \end{subfigure}
  ~
  \begin{subfigure}{0.23\textwidth}
    \includegraphics[width=\textwidth]{figs/syn/{All.Synthetic.real.d100.S1000}.png}
    \caption{$\Real^{100}$}
    \label{afig:bcs-r100}
  \end{subfigure}
  \caption{Performance of \fbfc/$\fbfc^*$ and baselines relative to
    the \knnc performance on {\em synthetic data}. The $10$-fold
    cross-validated accuracy is considered for each of the data
    sets. The box-plots correspond to the relative difference ({\em
      lower is better}) aggregated over 30 repetitions (see text for
    details).  The red dashed line corresponds to matching \knnc
    performance.}
  \label{afig:bcomp-syn}
\end{figure}

We also study the effect of the number of non-zeros $b < d$ in the binary data on the performance of \fbfc/$\fbfc^*$ and baselines (Figure~\ref{afig:bcomp-syn-bin-nnz}). The results indicate that, for fixed data dimensionality $d$, the relative performance of \fbfc (and variants) is not significantly affected by the choice of $b < d$. \cc is also robust to changes in $b$. The performance of \sbfc seems to improve with increasing $b$ while the opposite behavior is seen for \ccone, \lr and \mlpc.
\begin{figure*}[htb]
  \vskip 0.2in
  \begin{center}
    \centerline{\includegraphics[width=\textwidth]{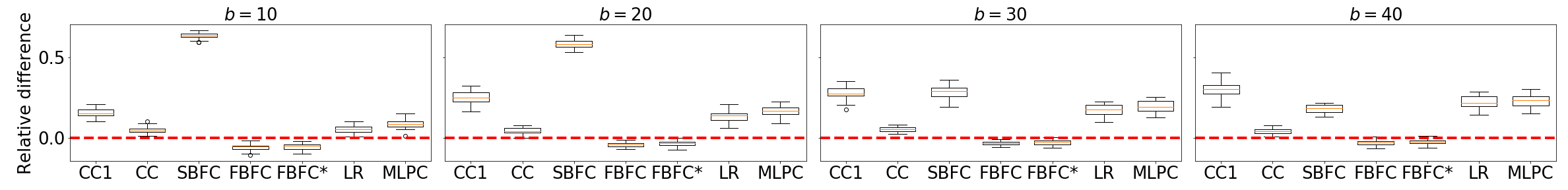}}
    \caption{Performance of \fbfc/$\fbfc^*$ and baselines relative to the performance of \knnc with varying number of non-zeros $b$ per point for fixed $d=100$. All methods undergo a hyper-parameter optimization and the best performance ($10$-fold cross-validation accuracy) is considered for each of the data sets. The boxplots corresponds to the 30 repetitions (in the form of 30 different synthetic data sets per experimental setting). A relative difference of $0$ implies matching the \knnc ({\em lower is better}) -- the red dashed line corresponds to \knnc performance.}
    \label{afig:bcomp-syn-bin-nnz}
  \end{center}
  \vskip -0.2in
\end{figure*}

\subsection{Additional details for OpenML data} \label{asec:expts:openml}

We consider two sets of OpenML data sets utilizing the following query for OpenML classification data sets with no categorical and missing features with (i) \texttt{min\_dim = 11}, \texttt{max\_dim = 101}, \texttt{max\_rows = 50000}, and (ii)  \texttt{min\_dim = 102}, \texttt{max\_dim = 1025}, \texttt{max\_rows = 10000}, leading to $79$ and $14$ data sets respectively where there were no issues with the data retrieval and the processing of the data with \sklearn operators.

\paragraph{OpenML query for data sets.}
\begin{footnotesize}
\begin{lstlisting}[language=Python]
from openml.datasets import list_datasets, get_dataset
openml_df = list_datasets(output_format='dataframe')
val_dsets = openml_df.query(
    'NumberOfInstancesWithMissingValues == 0 & '
    'NumberOfMissingValues == 0 & '
    'NumberOfClasses > 1 & '
    'NumberOfClasses <= 30 & '
    'NumberOfSymbolicFeatures == 1 & '
    'NumberOfInstances > 999 &'
    'NumberOfFeatures >= min_dim &'
    'NumberOfFeatures <= max_dim &'
    'NumberOfInstances <= max_rows'
)[[
    'name', 'did', 'NumberOfClasses',
    'NumberOfInstances', 'NumberOfFeatures'
]]
\end{lstlisting}
\end{footnotesize}


\end{document}